%% file: ms.tex
\documentclass[twoside]{article}

\usepackage[accepted]{aistats2022}
\usepackage[round]{natbib}

\usepackage[utf8]{inputenc} % allow utf-8 input
\usepackage[T1]{fontenc}    % use 8-bit T1 fonts
\usepackage{color,xcolor}
\definecolor{mydarkblue}{rgb}{0,0.08,0.45}
\usepackage[colorlinks=true,
    linkcolor=mydarkblue,
    citecolor=mydarkblue,
    filecolor=mydarkblue,
    urlcolor=mydarkblue]{hyperref}
\usepackage{url}            % simple URL typesetting
\usepackage{booktabs}       % professional-quality tables
\usepackage{nicefrac}       % compact symbols for 1/2, etc.
\usepackage{microtype}      % microtypography
\usepackage{amsmath,amsfonts,amssymb}
\usepackage{subfig}
\usepackage{mathtools}
\usepackage{enumitem}
\usepackage[page]{appendix}
\usepackage{algorithm}
\usepackage{algpseudocode}

\usepackage{amsthm}
\usepackage{thmtools,thm-restate}
\newtheorem{theorem}{Theorem}

\newtheorem{corollary}{Corollary}
\newtheorem{assumption}{Assumption}
\newtheorem{definition}{Definition}

\newtheorem{remark}{Remark}
\DeclareMathOperator*{\argmin}{arg\,min}
\DeclareMathOperator{\Tr}{tr}
\DeclareMathOperator{\grad}{grad}
\DeclareMathOperator{\relu}{ReLU}
\newcommand*{\T}{\mathsf{T}}
\newcommand*{\PA}{\mathsf{PA}}
\newcommand{\MLP}{\mathsf{MLP}}

\makeatletter
\def\th@plain{%
  \thm@notefont{}% same as heading font
  \itshape % body font
}
\def\th@definition{%
  \thm@notefont{}% same as heading font
  \normalfont % body font
}

\makeatother

\begin{document}

% If your paper is accepted and the title of your paper is very long,
% the style will print as headings an error message. Use the following
% command to supply a shorter title of your paper so that it can be
% used as headings.
%
%\runningtitle{I use this title instead because the last one was very long}

% If your paper is accepted and the number of authors is large, the
% style will print as headings an error message. Use the following
% command to supply a shorter version of the authors names so that
% they can be used as headings (for example, use only the surnames)
%
\runningauthor{Ignavier Ng, Sébastien Lachapelle, Nan Rosemary Ke, Simon Lacoste-Julien, Kun Zhang}

\twocolumn[

\aistatstitle{On the Convergence of Continuous Constrained Optimization for Structure Learning}

\aistatsauthor{ Ignavier Ng$^1$, Sébastien Lachapelle$^{2}$, Nan Rosemary Ke$^{3}$, Simon Lacoste-Julien$^{2,4}$, Kun Zhang$^{1,5}$ }

\aistatsaddress{
$^1$\,Carnegie Mellon University\\
$^2$\,Mila, Université de Montréal\\
$^3$\,DeepMind\\
$^4$\,Canada CIFAR AI Chair\\
$^5$\,Mohamed bin Zayed University of Artificial Intelligence
} 
]

\input{sections/0abstract}
\input{sections/1introduction}
\input{sections/2background}
\input{sections/3methodology}
\input{sections/4experiments}
\input{sections/5conclusion}
\input{sections/6acknowledgment}

\bibliographystyle{abbrvnat}
\bibliography{ms}

%%%%%%%%%%%%%%%%%%%%%%%%%%%%%%%%%%%
%%%%%% SUPPLEMENT (OPTIONAL) %%%%%%
%%%%%%%%%%%%%%%%%%%%%%%%%%%%%%%%%%%

\clearpage
\appendix

\thispagestyle{empty}
% For one-column format, uncomment the following:
\onecolumn \makesupplementtitle
% For two-column format, uncomment the following:
%\twocolumn[ \makesupplementtitle ]
\input{sections/7appendix}

\end{document}

%% file: sections/0abstract.tex
\begin{abstract}
Recently, structure learning of directed acyclic graphs (DAGs) has been formulated as a continuous optimization problem by leveraging an algebraic characterization of acyclicity. The constrained problem is solved using the augmented Lagrangian method (ALM) which is often preferred to the quadratic penalty method (QPM) by virtue of its standard convergence result that does not require the penalty coefficient to go to infinity, hence avoiding ill-conditioning. However, the convergence properties of these methods for structure learning, including whether they are guaranteed to return a DAG solution, remain unclear, which might limit their practical applications. In this work, we examine the convergence of ALM and QPM for structure learning in the linear, nonlinear, and confounded cases. We show that the standard convergence result of ALM does not hold in these settings, and demonstrate empirically that its behavior is akin to that of the QPM which is prone to ill-conditioning. We further establish the convergence guarantee of QPM to a DAG solution, under mild conditions. Lastly, we connect our theoretical results with existing approaches to help resolve the convergence issue, and verify our findings in light of an empirical comparison of them.
\end{abstract}

%% file: sections/1introduction.tex
\section{Introduction}\label{sec:intro}
Structure learning of directed acyclic graphs (DAGs) is a fundamental problem in many scientific endeavors, such as biology \citep{Sachs2005causal} and economics \citep{Koller09probabilistic}. Traditionally, score-based structure learning methods cast the problem into a discrete optimization program using a predefined score function. Most of these methods, such as GES \citep{Chickering2002optimal}, involve local heuristics owing to the large search space of graphs \citep{Chickering1996learning}.

A recent work by \citet{Zheng2018notears} has reformulated score-based learning of linear DAGs as a continuous constrained optimization problem. At the heart of the method is an algebraic characterization of acyclicity expressed as a nonlinear constraint and used to minimize the least squares loss while enforcing acyclicity. In the context of structure learning, various works have adopted this continuous constrained formulation to support linear non-Gaussian models \citep{Zheng2020thesis}, nonlinear models \citep{Yu19daggnn, Ng2019graph, Lachapelle2020grandag, Zheng2020learning,Gao2021daggan,Ng2022masked,Geffner2022fcause}, time series \citep{Pamfil2020dynotears,Sun2021ntsnotears,Hsieh2021srvarm}, unobserved confounding \citep{Bhattacharya2020differentiable,Bellot2021deconfounded}, interventional data \citep{brouillard2020differentiable,Faria2022differentiable}, multi-domain data \citep{Zeng2020causal}, mixed data \citep{Zeng2022causal}, low rank DAGs \citep{Fang2020low}, incomplete data \citep{Wang2020causal}, prior knowledge \citep{Cai2021anoce}, federated learning \citep{Ng2022towards,Gao2021federated}, and multi-task learning \citep{Chen2021multitask}. The continuous constrained formulation has also been applied to other domains, e.g., reinforcement learning \citep{Pruthi2020structure,Ruan2022gcs}, normalizing flows \citep{Wehenkel2020graphical,Dai2022graph}, domain adaptation \citep{Yang2021learning}, recommendation system \citep{Wang2022sequential}, and computer vision \citep{Cui2020knowledge,Yang2021causalvae,Zhang2021acre,Zhang2022cmgan}.

Like in the original work, most of these extensions rely on the \emph{augmented Lagrangian method} (ALM) \citep{Bertsekas1982constrained, Bertsekas1999nonlinear} to solve the continuous constrained optimization problem. This choice of algorithm was originally motivated by the convergence result of ALM, which, unlike the classical \emph{quadratic penalty method} (QPM) \citep{Powell1969nonlinear, Fletcher1987practical}, does not require increasing the penalty coefficient to infinity \citep[Proposition~3]{Zheng2018notears}. Despite abundant extensions and applications of the continuous constrained formulation, it remains unclear whether the required conditions for the standard convergence result of ALM, or more specifically, the regularity conditions, are satisfied in these settings, and whether the continuous constrained formulation is guaranteed to converge to a DAG solution, which is a key to structure learning.

{\bf Contributions.} \ \ \ 
We examine the convergence properties of ALM and QPM for structure learning in the linear, nonlinear, and confounded cases. We conclude (i) that, unfortunately, the conditions behind standard convergence result of ALM are not satisfied in these settings, (ii) that, furthermore, the empirical behavior of ALM is similar to QPM that requires the penalty coefficient to go to infinity and is prone to ill-conditioning, and (iii) that, interestingly, QPM is guaranteed to converge to a DAG solution, under mild conditions. We then provide the implications of our theoretical results for existing approaches to help resolve the convergence issue, with an empirical comparison of them that verifies our findings and makes them more intuitive.

We note that the problem has received considerable attention. For instance, \citet{Wei2020nofears} studied a related problem, involving the regularity conditions of the continuous constrained formulation. It is worth mentioning that our contributions are different and more complete in terms of the technical development and results. Specifically, \citet{Wei2020nofears} focused on the Karush-Kuhn-Tucker (KKT) conditions, while our study focuses on the convergence of specific constrained optimization methods (i.e., ALM and QPM), which provides practical insight for solving the optimization problem. Furthermore, they focused on the linear case, while our results also apply to the nonlinear and confounded cases, under the same umbrella.

{\bf Organization of the paper.} \ \ \ 
We give an overview of score-based learning, continuous constrained formulation of structure learning, and the ALM in Section \ref{sec:background}. In Section \ref{sec:optimization}, we examine the convergence of ALM and QPM for structure learning in the linear and nonlinear cases. We extend the results to the confounded case in Section \ref{sec:with_confounding}, and connect them with different approaches to resolve the convergence issue in Section \ref{sec:resolve_convergence}. We provide empirical studies in Section \ref{sec:exp} to verify our results, and conclude our work in Section \ref{sec:conclusion}.

%% file: sections/2background.tex
\section{Background}\label{sec:background}
We provide a brief review of score-based structure learning, the NOTEARS method \citep{Zheng2018notears} and the standard convergence result of ALM.

\subsection{Score-Based Structure Learning}
Structure learning refers to the problem of learning a graphical structure (in our case a DAG) from data. Given the random vector $X=(X_1, \dots, X_d)$ consisting of $d$ random variables, we assume that the corresponding design matrix $\mathbf{X} =[\mathbf{X}_1|\cdots|\mathbf{X}_d]\in \mathbb{R}^{n \times d}$ is generated from a joint distribution $P(X)$ (with a density $p(x)$) that is Markov with respect to the ground truth DAG $\mathcal{G}$ and can be factorized as $p(x) = \prod_{i=1}^d p_i(x_i|x_{\PA_i^\mathcal{G}})$, where $\PA_i^\mathcal{G}$ designates the set of parents of $X_i$ in the DAG $\mathcal{G}$. In general, the underlying DAG is only identifiable up to Markov equivalence under the faithfulness \citep{Spirtes2000causation} or the sparsest Markov representation assumption \citep{Raskutti2018learning}. Under certain assumptions on the data distribution, the DAG $\mathcal{G}$ is fully identifiable, such as the linear non-Gaussian model \citep{Shimizu2006lingam}, linear Gaussian model with equal noise variances \citep{Peters2013identifiability}, nonlinear additive noise model \citep{Hoyer2009nonlinear,Peters2014causal}, and post-nonlinear model \citep{Zhang2009identifiability}.

To recover the structure $\mathcal{G}$ or its Markov equivalence class, a major class of structure learning methods are the score-based methods that solves an optimization problem over the space of graphs using some goodness-of-fit measure with a sparsity regularization term. Some examples include GES \citep{Chickering2002optimal}, $\ell_0$-regularized likelihood \citep{Van2013ell_0}, integer linear programming \citep{pmlr-v9-jaakkola10a, cussens2012bayesian}, dynamic programming \citep{Koivisto2004exact,Ott2004finding,Singh2005finding}, and A* \citep{Yuan2013learning}. Most of these methods tackle the structure search problem in its natural discrete form.

\subsection{Continuous Constrained Optimization for Structure Learning}\label{sec:notears_background}
{\bf NOTEARS.} \ \ \ 
\citet{Zheng2018notears} proposed a continuous constrained formulation for score-based learning of linear DAGs. In particular, the linear DAG model is equivalently represented by the linear structural equation model (SEM)
\begin{equation}\label{eq:linear_sem}
X = W^\T X + N,
\end{equation}
where $W$ is the weighted adjacency matrix of a DAG, and $N$ is a noise vector characterized by the noise covariance matrix $\Omega$. We assume here that the elements of $N$ are mutually independent, and there is no unobserved confounder. Denoting by $\odot$ the element-wise product, and by $e^M$  the matrix exponential of a square matrix $M$, the authors have shown that $\Tr(e^{W\odot W}) - d = 0$ holds if and only if $W$ represents a DAG. The resulting continuous constrained optimization problem is
\begin{align*}
	\min_{W\in\mathbb{R}^{d\times d}}\ \ &\frac{1}{2n}\|\mathbf{X} -  \mathbf{X}W\|_2^2 + \lambda\|W\|_1\\
	\text{subject to}\ \ &\Tr(e^{W\odot W}) - d = 0,
\end{align*}
where $\|\cdot\|_2$ and $\|\cdot\|_1$ denote the element-wise $\ell_2$ and $\ell_1$ norms, respectively, and $\lambda$ is the regularization coefficient. Here, $(1/2n)\|\mathbf{X} -  \mathbf{X}W\|_2^2$ is the least squares objective and is equal, up to a constant, to the log-likelihood of linear Gaussian DAGs assuming equal noise variances. The $\ell_1$ regularization term $\|W\|_1$ is useful for enforcing sparsity on the matrix $W$.

{\bf NOTEARS-MLP.} \ \ \ 
To generalize the above formulation to the nonlinear case, \citet{Zheng2020learning} used multi-layer perceptrons (MLPs) to model nonlinear relationships. For each variable $X_i$, let $\MLP(u; A_i)$ be the corresponding MLP with input row vector $u$, $\ell$ layers, weights $A_i=(A_i^{(1)}, \dots, A_i^{(\ell)})$, and element-wise activation function $\sigma(\cdot)$, defined as follows:
\begin{equation*}
\begin{gathered}
    \MLP(u; A_i) = \sigma(\cdots \sigma(\sigma(uA_i^{(1)})A_i^{(2)})\cdots)A_i^{(\ell)},\\
    A_i^{(t)}\in\mathbb{R}^{s_{t-1}\times s_{t}}, \quad s_0=d, \ s_\ell=1.
\end{gathered}
\end{equation*}
Let $A=(A_1,\dots,A_d)$ denote the weights of MLPs corresponding to all variables. The authors defined an equivalent adjacency matrix $(B(A))_{ji}= \| j\text{th-row}(A_{i}^{(1)} ) \|_2$ and proposed to solve the optimization problem
\begin{align*}
	\min_{A}\ \ &\frac{1}{n}\sum_{i=1}^{d}\|\mathbf{X}_i - \MLP(\mathbf{X};A_i)\|_2^2 + \lambda\|A_i^{(1)}\|_1\\
	\text{subject to}\ \ &\Tr(e^{B(A)\odot B(A)}) - d = 0.
\end{align*}
As described in Section \ref{sec:intro}, there are several other extensions of NOTEARS to the nonlinear case that also adopt the continuous constrained formulation, e.g., DAG-GNN \citep{Yu19daggnn}, GraN-DAG \citep{Lachapelle2020grandag}, and MCSL \citep{Ng2022masked}. In this work we focus on NOTEARS-MLP for convergence analysis, since it is conceptually simple compared to the others. We leave the analysis for the others for future work.

{\bf Optimization.} \ \ \ 
The above optimization problems involve a hard DAG constraint and are solved via ALM, a general method for continuous constrained optimization, which we review next.

\subsection{Augmented Lagrangian Method}\label{sec:alm_background}
Consider the generic constrained optimization problem
\begin{equation}
	\min_{\theta \in \mathbb{R}^m}\ \ f(\theta) \ \ \text{subject to}\ \ h(\theta) = 0, \label{eq:generic-opt}
\end{equation}
where the functions $f:\mathbb{R}^m \rightarrow \mathbb{R}$ and $h:\mathbb{R}^m \rightarrow \mathbb{R}^p$ are both twice continuously differentiable.

The ALM transforms a constrained optimization problem like \eqref{eq:generic-opt} into a sequence of unconstrained ones with solutions converging to a solution of the original problem. The key idea is to combine the Lagrangian formulation with QPM, yielding an augmented problem
\begin{equation*}
	\min_{\theta \in \mathbb{R}^m}\ \ f(\theta)+\frac{\rho}{2}\|h(\theta)\|_2^2 \ \ \text{subject to}\ \ h(\theta) = 0, \label{eq:augmented-opt}
\end{equation*}
where $\rho>0$ is the penalty coefficient. The augmented Lagrangian function of the formulation above is
\begin{equation*}
L(\theta, \alpha;\rho) = f(\theta) + \alpha^\T h(\theta) + \frac{\rho}{2}\|h(\theta)\|_2^2,
\end{equation*}
where $\alpha \in \mathbb{R}^p$ is an estimate of the Lagrange multiplier. A version of the procedure is described in Algorithm \ref{alg:alm}, which is essentially based on dual ascent. The minimization problem of $L(\theta, \alpha;\rho)$ can sometimes be solved only to stationarity, if, for example, it is nonconvex, as is the case in the formulation of NOTEARS owing to the nonconvexity of its specific constraint function. 
\begin{algorithm}[t]
\caption{Augmented Lagrangian Method \citep[Framework~17.3]{Nocedal2006numerical}}
\label{alg:alm}
\begin{algorithmic}[1]
    \Require starting penalty coefficient $\rho_1 > 0$; starting Lagrange multiplier $\alpha_1$; multiplicative factor $\beta > 1$; reduction factor $\gamma < 1$; nonnegative sequence $\{\tau_k\}$; starting point $\theta_0$
    \For{$k=1,2,\ldots$}
        \State {\spaceskip  1.1em  \relax Find an approximate minimizer $\theta_k$ of}
        \Statex \qquad \quad {\spaceskip  0.66em  \relax $L(\cdot,\alpha_k;\rho_k)$, starting at point $\theta_{k-1}$, and}
        \Statex \qquad \quad terminating when $\|\nabla_\theta L(\theta,\alpha_k;\rho_k) \|_2 \leq \tau_k$
        \If {final convergence test satisfied}
            \State \textbf{stop} with approximate solution $\theta_k$
        \EndIf
        \State Update multiplier $\alpha_{k+1} = \alpha_{k}+ \rho_{k}h(\theta_k)$
        \If {$\|h(\theta_k)\|_2 > \gamma \|h(\theta_{k-1})\|_2$}
            \State Update penalty coefficient $\rho_{k+1} = \beta \rho_k$
        \Else
            \State Update penalty coefficient $\rho_{k+1} = \rho_k$
        \EndIf
    \EndFor
\end{algorithmic}
\end{algorithm}

Based on the procedure of ALM outlined above, we review one of its standard convergence results \citep{Bertsekas1982constrained, Bertsekas1999nonlinear, Nocedal2006numerical}. The following definition is required to state this result and is crucial to the contribution of our work.
\begin{definition}[Regular point]\label{def:regularity}
We say that a point $\theta^*$ is regular, or that it satisfies the linear independence constraint qualification (LICQ), if the rows of the Jacobian matrix of $h$ evaluated at $\theta^*$, $\nabla_\theta h(\theta^*) \in \mathbb{R}^{p \times m}$, are linearly independent.
\end{definition}
\begin{theorem}[{\citet[Theorem~17.5 \& 17.6]{Nocedal2006numerical}}]\label{thm:alm}
Let $\theta^*$ be a regular point of \eqref{eq:generic-opt} that satisfies the second-order sufficient conditions (see Appendix \ref{sec:optimality_condition}) with vector $\alpha^*$. Then there exist positive scalars $\bar{\rho}$ (sufficiently large), $\delta$, $\epsilon$, and $M$ such that for all $\alpha_k$ and $\rho_k$ satisfying
    \begin{equation*}\label{eq:alpha_bound}
    \|\alpha_k - \alpha^*\|_2 \leq \rho_k \delta, \quad \rho_k \geq \bar{\rho},
    \end{equation*}
    the problem
    \[
	\min_{\theta\in\mathbb{R}^m}\ \ L(\theta, \alpha_k; \rho_k)\ \ \text{\normalfont subject to}\ \ \|\theta - \theta^*\|_2 \leq \epsilon
    \]
    has a unique solution $\theta_k$. Moreover, we have
    \begin{equation}\label{eq:theta-ineq}
    \|\theta_k - \theta^*\|_2 \leq M\|\alpha_k - \alpha^*\|_2/\rho_k
    \end{equation}
    and
    \begin{equation}\label{eq:alpha-ineq}
        \|\alpha_{k+1} - \alpha^*\|_2 \leq M\|\alpha_{k} - \alpha^*\|_2/\rho_k \, ,
    \end{equation}
    where $\alpha_{k+1}= \alpha_{k}+ \rho_{k}h(\theta_k)$.
\end{theorem}
An important consequence of Theorem \ref{thm:alm} is that if the penalty coefficient $\rho_k$ is larger than both $\bar{\rho}$ and $\|\alpha_k - \alpha^*\|_2/\delta$, then the inequalities \eqref{eq:theta-ineq} and \eqref{eq:alpha-ineq} hold. If, in addition, $\rho_k > M $, then $\alpha_k \rightarrow \alpha^*$ by \eqref{eq:alpha-ineq} and $\theta_k \rightarrow \theta^*$ by \eqref{eq:theta-ineq}, \emph{without increasing the coefficient $\rho_k$ to infinity}. This property often motivates the usage of the ALM over the other approaches, e.g., QPM, for constrained optimization. Specifically, this was the original motivation provided by \citet[Proposition~3]{Zheng2018notears} for using the ALM. The reason is that the QPM requires bringing the penalty coefficient $\rho_k$ to infinity, which may lead to ill-conditioning issue when solving the minimization problem. In the next section, we examine the required conditions and show that Theorem \ref{thm:alm} does not apply to the continuous constrained formulation proposed by \citet{Zheng2018notears,Zheng2020learning}.

%% file: sections/3methodology.tex
\section{Convergence of the Continuous Constrained Optimization Methods}\label{sec:optimization}
In this section, we take a closer look at the convergence of ALM and QPM for learning DAGs in the linear and nonlinear cases. We assume that there is no unobserved confounder and focus on the formulation of NOTEARS and NOTEARS-MLP. We will show in Section \ref{sec:with_confounding} that our analysis generalizes to the confounded case. We consider the constrained optimization problem
\begin{equation}
	\min_{\theta}\ \ f(\theta)\ \ \text{subject to}\ \ h(B(\theta)) = 0, \label{eq:notears_general}
\end{equation}
where $f(\theta)$ is the objective function and $h(B(\theta))$ is a (scalar-valued) constraint function that enforces acyclicity on the (equivalent) weighted adjacency matrix $B(\theta)$. We assume here that the functions $f$, $h$, and $B$ are continuously differentiable. Specifically, the constraint term proposed by \citet{Zheng2018notears} is given by $h_\text{exp}(B(\theta))= \Tr(e^{B(\theta)\odot B(\theta)}) - d$. As described in Section \ref{sec:notears_background}, in the linear case, the parameter $\theta$ corresponds to the weighted adjacency matrix of the linear SEM, i.e., we have $\theta= W$ and $B(W)= W$. In the nonlinear case, the parameter $\theta$ corresponds to the weights of the MLPs. That is, we have $\theta= A$ and $(B(A))_{ji}= \| j\text{th-row}(A_{i}^{(1)} ) \|_2$. Hereafter we use $f(W)$ and $h_\text{exp}(W)$ to refer to the objective and DAG constraint term in the linear case, and $f(A)$ and $h_\text{exp}(B(A))$ to refer to those in the nonlinear case.

By assuming that function $f$ is continuously differentiable, our analysis does not consider the $\ell_1$ regularization for simplicity. In Section \ref{sec:exp}, we study empirically the constrained formulation with and without the excluded regularization term, and show that our analysis appears to generalize to the $\ell_1$-regularized case.

\subsection{Regularity of DAG Constraint Term}\label{sec:regularity}
To investigate whether Theorem \ref{thm:alm} applies to problem \eqref{eq:notears_general}, one has to first verify if the DAG constraint term $h(B(\theta))$ satisfies the regularity conditions. The following condition is required for our analysis.
\begin{assumption}\label{assump:dag_constraint}
The function $h(B(\theta)) = 0$ if and only if its gradient $\nabla_\theta h(B(\theta)) = 0$.
\end{assumption}
Both DAG constraint terms in the linear and nonlinear cases satisfy the assumption above, with a proof provided in Appendix \ref{sec:proof_constraint_notears}.
\begin{theorem}\label{thm:constraint_notears}
The functions $h_\text{exp}(W)$ and $h_\text{exp}(B(A))$ satisfy Assumption \ref{assump:dag_constraint}.
\end{theorem}
Assumption \ref{assump:dag_constraint} implies that the Jacobian matrix of function $h(B(\theta))$ (after reshaping) evaluated at any feasible point of problem \eqref{eq:notears_general} corresponds to a zero row vector, which is itself is not linearly independent and therefore leads to the following remark.
\begin{remark}\label{remark:alm}
If the function $h(B(\theta))$ satisfies Assumption \ref{assump:dag_constraint}, any feasible solution of problem \eqref{eq:notears_general} is not regular and Theorem \ref{thm:alm} does not apply.
\end{remark}
With Theorem \ref{thm:constraint_notears}, this shows that the advantage of ALM (illustrated by Theorem \ref{thm:alm}) does not apply to the DAG constraints developed by \citet{Zheng2018notears,Zheng2020learning} in the linear and nonlinear cases. Hence, we are left with no guarantee that the penalty coefficient $\rho$ does not have to go to infinity for ALM to converge. In Section \ref{sec:exp_alm_vs_qpm}, we show empirically that $\rho$ grows without converging just like it would in QPM that is prone to ill-conditioning, which we describe in the next section.

\subsection{Quadratic Penalty Method}\label{sec:qpm}
Apart from ALM, QPM is another method for solving the constrained optimization problem \eqref{eq:notears_general}, whose convergence property is studied in this section. We first define the quadratic penalty function
\begin{equation}\label{eq:qpm}
Q(\theta; \rho) = f(\theta) + \frac{\rho}{2} h(B(\theta))^2,
\end{equation}
and describe the procedure of QPM in Algorithm \ref{alg:qpm}. Note that it is essentially the same as ALM but without the Lagrangian part, and thus has a simpler procedure.

This approach adds a quadratic penalty term for the constraint violation to the objective $f(\theta)$. By gradually increasing the penalty coefficient $\rho$, we penalize the constraint violation with increasing severity. Therefore, it makes intuitive sense to think that the procedure converges to a feasible solution (i.e., a DAG solution) as we bring $\rho$ to infinity. However, this is not necessarily true: in general, Algorithm \ref{alg:qpm} returns only a stationary point of the quadratic penalty term $h(B(\theta))^2$  \citep[Theorem~17.2]{Nocedal2006numerical}. Fortunately, if the DAG constraint term $h(B(\theta))$ satisfies Assumption \ref{assump:dag_constraint}, the procedure is guaranteed to converge to a feasible solution, under mild conditions, formally stated in Theorem \ref{thm:qpm}. The proof is provided in Appendix \ref{sec:proof_qpm_convergence}. Note that this theorem and its proof are adapted from Theorem $17.2$ in \citet{Nocedal2006numerical}.

\begin{algorithm}[t]
\caption{Quadratic Penalty Method \citep[Framework~17.1]{Nocedal2006numerical}}
\label{alg:qpm}
\begin{algorithmic}[1]
    \Require starting penalty coefficient $\rho_1 > 0$; multiplicative factor $\beta > 1$; nonnegative sequence $\{\tau_k\}$; starting point $\theta_0$
    \For{$k=1,2,\ldots$}
        \State Find an approximate minimizer $\theta_k$ of $Q(\cdot;\rho_k)$,
        \Statex \qquad \quad {\spaceskip  0.64em  \relax starting at point $\theta_{k-1}$, and terminating}
        \Statex \qquad \quad when $\|\nabla_\theta Q(\theta;\rho_k) \|_2 \leq \tau_k$
        \If {final convergence test satisfied}
            \State \textbf{stop} with approximate solution $\theta_k$
        \EndIf
        \State Update penalty coefficient $\rho_{k+1} = \beta \rho_k$
    \EndFor
\end{algorithmic}
\end{algorithm}

\begin{theorem}\label{thm:qpm}
Suppose in Algorithm \ref{alg:qpm} that the penalty coefficients satisfy $\rho_k \rightarrow \infty$ and the sequence of nonnegative tolerances $\{\tau_k\}$ is bounded.\footnote{A stricter condition $\tau_k \rightarrow 0$ is often used in the analysis of QPM \citep[Theorem~17.2]{Nocedal2006numerical} but is not required here.} Suppose also that the function $h(B(\theta))$ satisfies Assumption \ref{assump:dag_constraint}. Then every limit point $\theta^*$ of the sequence $\{\theta_k\}$ is feasible.
\end{theorem}
\begin{remark}
With the constraint terms $h_\text{exp}(W)$ and $h_\text{exp}(B(A))$, Theorems \ref{thm:qpm} and \ref{thm:constraint_notears} guarantee that, under mild conditions, Algorithm \ref{alg:qpm} returns a DAG solution as $\rho_k \rightarrow \infty$ based on inexact minimizations of $Q(\cdot;\rho_k)$.
\end{remark}
Although the standard convergence result of ALM (i.e., Theorem \ref{thm:alm}) does not hold as the DAG constraint terms proposed by \citet{Zheng2018notears,Zheng2020learning} satisfy Assumption \ref{assump:dag_constraint}, this property ensures that QPM returns a DAG, which is indeed a key to structure learning. The remark above also explains why the implementations of ALM with these two constraints often return DAG solutions in practice (after thresholding). Furthermore, if Assumption \ref{assump:dag_constraint} is satisfied, Theorem \ref{thm:qpm} verifies that one can directly use the value of DAG constraint term as an indicator for the final convergence test in Algorithm \ref{alg:qpm}, i.e., $h(B(\theta_k))\leq\epsilon$ with $\epsilon>0$ being a small tolerance. It is worth noting that this convergence test has been adopted in the current implementation of NOTEARS \citep{Zheng2018notears} and most of its extensions. 

{\bf Practical issue.} \ \ \ 
In practice, one is, at most, only able to increase the penalty coefficient $\rho$ to a very large value, e.g., of order $10^{16}$. Therefore, the final solution can only satisfy $h(B(\theta))\leq\epsilon$ up to numerical precision with $\epsilon>0$ being a small tolerance, e.g., of order $10^{-8}$. In this case, the solution may contain many entries close to zero and does not correspond exactly to a DAG. Following \citet{Zheng2018notears}, a thresholding step on the estimated entries is needed to convert the solution into a DAG; the experiments in Section \ref{sec:exp_alm_vs_qpm} suggest that a small threshold (e.g., $0.05$) suffices. However, a moderately large threshold (e.g., $0.3$) can still be useful for reducing the false discoveries.

\subsection{Other DAG Constraint Term}
Apart from the matrix exponential term proposed by \citet{Zheng2018notears}, \citet{Yu19daggnn} developed a polynomial alternative that may have better numerical stability with a proper choice of $\mu>0$:
\[h_\text{poly}(B(\theta))= \Tr\left((I + \mu B(\theta)\odot B(\theta))^d\right)- d.\]
Our analysis generalizes to the above constraint term.
\begin{corollary}\label{cor:dag_constraint_notears}
The functions $h_\text{poly}(W)$ and $h_\text{poly}(B(A))$ satisfy Assumption \ref{assump:dag_constraint}.
\end{corollary}

\section{With Unobserved Confounding}\label{sec:with_confounding}
We study whether our convergence analysis is applicable to the confounded case. Recently, \citet{Bhattacharya2020differentiable} applied the continuous constrained formulation proposed by \citet{Zheng2018notears} to estimate structures with unobserved confounding, by deriving algebraic characterizations for different classes of acyclic directed mixed graphs (ADMGs), i.e., ancestral, arid, and bow-free graphs. Here we consider the bow-free graphs that are the least restrictive for our further analysis. Note that a bow-free ADMG refers to an ADMG in which the directed and bidirected edges do not both appear for any pair of vertices.

Denote by $W$ and $\Omega$ the weighted adjacency matrix and noise covariance matrix of the linear SEM defined in Eq. \eqref{eq:linear_sem}. In the confounded case, there exist unobserved variables that are parents of more than one observed variable, implying that the noise terms are correlated \citep{Pearl2009causality}. Since $\Omega$ is symmetric, we have $X_i\leftrightarrow X_j$, $i\neq j$ in the ADMG if and only if $\Omega_{ji}\neq 0$. In other words, the weighted adjacency matrix $W$ and noise covariance matrix $\Omega$ represent the directed and bidirected edges in the ADMG, respectively. \citet{Bhattacharya2020differentiable} adopted the ALM to solve the constrained optimization problem \eqref{eq:notears_general} with the approximate BIC score \citep{Su2016sparse}, where $\theta$ corresponds to the parameters $W$ and $\Omega$. In this case, the algebraic constraint term of bow-free ADMGs is given by
\[
h_\text{bf}(W, \Omega) = \Tr(e^{W\odot W}) - d + \operatorname{sum}(W\odot W \odot \Omega \odot \Omega).
\]
The authors have shown that $h_\text{bf}(W, \Omega) = 0$ if and only if the ADMG defined by $W$ and $\Omega$ corresponds to a bow-free graph. To study the convergence property of ALM in this case, we have the following result regarding its regularity, with a proof given in Appendix \ref{sec:proof_constraint_bowfree}.
\begin{theorem}\label{thm:constraint_bowfree}
$h_\text{bf}(W, \Omega)$ satisfies Assumption \ref{assump:dag_constraint}.
\end{theorem}
Similar result also holds for the polynomial constraint term proposed by \citet{Yu19daggnn}.
\begin{corollary}\label{cor:dag_constraint_bowfree}
Theorems \ref{thm:constraint_bowfree} holds if the matrix exponential $e^{W\odot W}$ in the function $h_\text{bf}(W, \Omega)$ is replaced with the matrix polynomial $(I + \mu W\odot W)^d$ for any $\mu > 0$.
\end{corollary}
As a consequence, similar to the setting of NOTEARS and NOTEARS-MLP studied in Section \ref{sec:optimization}, Remark \ref{remark:alm} indicates there is no guarantee such that the penalty coefficient $\rho$ does not have to go to infinity for ALM to converge. Fortunately, with Theorem \ref{thm:qpm}, using QPM to solve the constrained optimization problem is guaranteed to return a solution that satisfies $h_\text{bf}(W, \Omega)\leq\epsilon$ up to numerical precision, under mild conditions.

\section{Resolving the Convergence Issue}\label{sec:resolve_convergence}
The analysis in Sections \ref{sec:optimization} and \ref{sec:with_confounding} implies that the advantage of ALM illustrated by Theorem \ref{thm:alm} does not hold in the continuous constrained formulation for structure learning. Specifically, it is not guaranteed that the penalty coefficient does not have to be increased indefinitely for the convergence of ALM. The experiments in Section \ref{sec:exp_alm_vs_qpm} verify this study and show empirically that ALM requires increasing the coefficient to a very large value to converge to a DAG solution, similar to QPM. This is known to cause numerical difficulties and ill-conditioning issues on the objective landscape \citep{Bertsekas1999nonlinear,Nocedal2006numerical}. The reason is that when the penalty term is large, the Hessian matrix is ill-conditioned and has a high condition number. In this case, the function contour is stretched out, and the gradients may not be the best direction to descend to the minimum, leading to a zigzag path. This is illustrated by a bivariate example in Appendix \ref{sec:bivariate_example}.

In light of our theoretical results, we give a brief overview on different approaches that help resolve the convergence issue, and provide an empirical comparison of them in Sections \ref{sec:exp_optimizer} and \ref{sec:exp_resolve_convergence} to illustrate our findings. In particular, the first approach uses a second-order method that is less susceptible to ill-conditioning, while the second approach devises an alternative algebraic DAG constraint with a local search procedure. The last two approaches adopt a different unconstrained formulation, thus avoiding the convergence issue caused by the hard DAG constraint in problem \eqref{eq:notears_general}.

It is worth noting that the effectiveness of some of the approaches below have been studied separately in the papers that proposed them. To the best of our knowledge, some of these approaches have not been connected with the convergence issue of NOTEARS, which is the focus of this section as well as Sections \ref{sec:exp_optimizer} and \ref{sec:exp_resolve_convergence} . Doing so allows one to understand which approach better resolves the convergence issue.

{\bf Second-order method.} \ \ \ 
As pointed out by \citet{Bertsekas1999nonlinear,Antoniou2007practical,Bottou2018optimization}, one may consider using a second-order method such as the quasi-Newton method \citep{Nocedal2006numerical} that handles ill-conditioning better by incorporating curvature information through approximations of the Hessian matrix. This is consistent with the recent works \citep{Zheng2018notears,Zheng2020learning,Pamfil2020dynotears} that adopt L-BFGS \citep{Byrd2003lbfgs} to solve the optimization subproblems. However, the original motivation of using quasi-Newton method was mainly about efficiency consideration, or, specifically, to reduce the number of evaluations of the matrix exponential that takes $O(d^3)$ cost \citep{Mohy2009scaling,Zheng2018notears}. We note here that another key advantage of using L-BFGS is, interestingly, to help resolve ill-conditioning issue, as verified by the experiments in Section \ref{sec:exp_optimizer} and a bivariate example in Appendix \ref{sec:bivariate_example}.

{\bf Absolute value adjacency matrix and KKT-informed local search.} \ \ \ 
In contrast to the quadratic adjacency matrix $B(\theta)\odot B(\theta)$ in the DAG constraint term $h_\text{exp}(B(\theta))$, the Abs-KKTS method proposed by \citet{Wei2020nofears} adopts an absolute value adjacency matrix given by $h_\text{exp}^\prime(B(\theta))= \Tr(e^{|B(\theta)|}) - d$, where $|\cdot|$ denotes the  element-wise absolute value of a matrix, together with a local search procedure informed by the KKT conditions as a post-processing step. The resulting procedure returns a solution that satisfies the KKT conditions, and leads to improvement in the structure learning performance of NOTEARS.

{\bf Soft constraints.} \ \ \ 
\citet{Ng2020role} showed that soft sparsity and DAG constraints suffice to asymptotically estimate a DAG equivalent to the true DAG under mild conditions when using the likelihood of linear Gaussian directed graphical models (possibly cyclic) as the objective function instead of the least squares loss that corresponds to the likelihood of linear Gaussian DAGs. This gives rise to the following \emph{unconstrained} optimization problem in the case of equal noise variances, denoted as GOLEM-EV:
\begin{align*}
\min_{W\in\mathbb{R}^{d\times d}}\ \ &\frac{d}{2}\log\|\mathbf{X} -  \mathbf{X}W\|_2^2 - \log|\det(I - W)| \\
& \quad + \lambda_1\|W\|_1 + \lambda_2 h_\text{exp}(W),
\end{align*}
where $\lambda_1$ and $\lambda_2$ are the regularization coefficients. With this unconstrained formulation, the ill-conditioning issue caused by the hard DAG constraint in problem \eqref{eq:notears_general} can be completely avoided, since one is able to directly solve the above problem using continuous optimization, without the need of any constrained optimization method like ALM or QPM.

{\bf Direct optimization in DAG space.} \ \ \ 
\citet{Yu2021nocurl} developed an algebraic representation of DAGs based on graph Hodge theory \citep{Jiang2011statistical,Jorgen2009digraphs}, and showed that one can directly perform continuous optimization in the space of all possible DAGs without relying on a hard DAG constraint. Similar to GOLEM-EV, the constrained optimization problem \eqref{eq:notears_general} can be reformulated in the linear case as an \emph{unconstrained} one
\begin{equation}\label{eq:nocurl}
(U^*,p^*)=\argmin_{U\in S, p\in \mathbb{R}^d}\ \ f(U\odot \relu(\grad(p))),
\end{equation}
where $S$ refers to the space of all $d\times d$ skew-symmetric matrices, $(\grad(p))_{ji}=p_i - p_j$ denotes the gradient flow defined on the nodes of a graph \citep{Lim2015hodge}, and $(\relu(M))_{ji}=\max(0, M_{ji})$ denotes the rectified linear unit function \citep{Nair2010relu} of a square matrix $M$. The final solution is given by $W^*=U^*\odot \relu(\grad(p^*))$, whose nonzero entries are guaranteed to represent a DAG \citep[Theorem~3.5]{Yu2021nocurl}. Since the objective function is highly nonconvex, randomly initializing $U$ and $p$ may lead to a stationary point far from the global optimum. The authors thus proposed a two-step procedure, denoted as NoCurl, that first obtains a rough estimate of the solution by solving the subproblem of NOTEARS either once or twice with a slightly large penalty coefficient (e.g., $\rho=10^3$), and uses that estimate to compute the initialization of $U$ and $p$ for problem \eqref{eq:nocurl}. Similar to second-order method, this method was originally motivated by efficiency consideration. Since it avoids the hard DAG constraint and accordingly does not require a large penalty coefficient, we note that this method can also help avoid the ill-conditioning issue.

%% file: sections/4experiments.tex
\section{Experiments}\label{sec:exp}
We conduct experiments on the structure learning tasks and take a closer look at the optimization processes to verify our study. In Section \ref{sec:exp_alm_vs_qpm}, we demonstrate that ALM behaves similarly to QPM, both of which converge to an approximately DAG solution when the penalty coefficients are very large. We compare the ability of different optimization algorithms to handle ill-conditioning in Section \ref{sec:exp_optimizer}, and the other approaches to help resolve the convergence issue in Section \ref{sec:exp_resolve_convergence}.

{\bf Methods.} \ \ \ 
We experiment with both NOTEARS and NOTEARS-MLP. We also consider their variants with the $\ell_1$ regularization term, denoted as NOTEARS-L1 and NOTEARS-MLP-L1, respectively.

{\bf Implementations.} \ \ \ 
Our implementations are based on the code\footnote{\url{https://github.com/xunzheng/notears}} released by \citet{Zheng2018notears, Zheng2020learning} with the DAG constraint term $h_\text{exp}(B)$. We also use the least squares objective and default hyperparameters in our experiments. Unless otherwise stated, we employ the L-BFGS algorithm \citep{Byrd2003lbfgs} to solve each subproblem and a threshold of $0.3$ for post-processing. In the linear case, we use a pre-processing step to center the data by subtracting the mean of each variable from the samples $\mathbf{X}$. The code is available at \url{https://github.com/ignavierng/notears-convergence}.

{\bf Simulations.} \ \ \ 
We simulate the ground truth DAGs using the Erd\"{o}s--R\'{e}nyi \citep{Erdos1959random} or scale-free \citep{Barabasi1999emergence} model with $kd$ edges on average, denoted as ER$k$ or SF$k$, respectively. Unless otherwise stated, based on the graph sizes $d\in\{10,20,50,100\}$ and different data generating procedures, we generate $1000$ samples with standard Gaussian noises. For NOTEARS and NOTEARS-L1, we simulate the linear DAG model with edge weights sampled uniformly from $[-2, -0.5] \cup  [0.5, 2]$, similar to \citep{Zheng2018notears}. For the nonlinear variants NOTEARS-MLP and NOTEARS-MLP-L1, we consider the data generating procedure used by \citet{Zheng2020learning}, where each function is sampled from a Gaussian process with RBF kernel of bandwidth one. Both data models are known to be fully identifiable \citep{Peters2013identifiability, Peters2014causal}.

{\bf Metrics.} \ \ \
We report the structural Hamming distance (SHD), structural intervention distance (SID) \citep{Peters2013structural} and true positive rate (TPR), averaged over $30$ random trials.

\begin{figure*}[!t]
\centering
\includegraphics[width=0.99\textwidth]{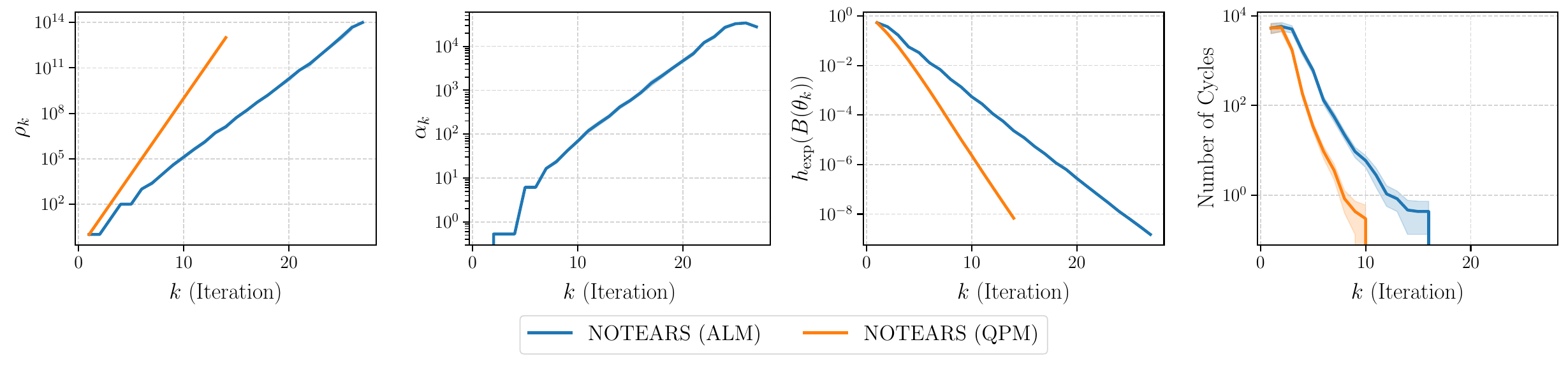}
\caption{Optimization processes of NOTEARS using ALM and QPM on synthetic data. The ground truths are $10$-node ER1 graphs and the sample size is $n=1000$. Each data point corresponds to the $k$-th iteration.
}
\label{fig:ALM_vs_QPM_optz_process_1}
\end{figure*}

\subsection{ALM Behaves Similarly to QPM}\label{sec:exp_alm_vs_qpm}
We conduct experiments to show that ALM behaves similarly to QPM in the context of structure learning, and that both of them converge to a DAG solution. Our goal here is not to show that QPM performs better than ALM, but rather to study their empirical behavior.

We first take a closer look at the optimization processes of ALM and QPM on the $10$-node ER1 graphs. Figure \ref{fig:ALM_vs_QPM_optz_process_1} depicts the penalty coefficient $\rho_k$, estimate of Lagrange multiplier $\alpha_k$ (only for ALM), value of DAG constraint term $h_\text{exp}(B(\theta_k))$, and number of cycles in the $k$-th iteration of the optimization for NOTEARS, while those for NOTEARS-L1, NOTEARS-MLP, and NOTEARS-MLP-L1 are visualized in Figure \ref{fig:ALM_vs_QPM_optz_process_2} in Appendix \ref{sec:supp_exp_results}. Note that we use a small threshold $0.05$ when computing the number of cycles. Complementing our study in Section \ref{sec:regularity}, ALM requires a very large coefficient $\rho_k$ to converge, similar to QPM, which suggests that they both behave similarly and that the standard convergence result of ALM appears to not hold here. On the other hand, when the penalty coefficient $\rho_k$ is very large, one observes that both ALM and QPM converge to a solution whose value of $h_\text{exp}(B(\theta_k))$ is very close to zero, i.e., smaller than $10^{-8}$, which yields a DAG solution after thresholding at $0.05$. Since one is not able to increase the penalty coefficient $\rho_k$ to infinity in practice, this serves as an empirical validation of Theorem \ref{thm:qpm}. Interestingly, Figures \ref{fig:ALM_vs_QPM_optz_process_1} and \ref{fig:ALM_vs_QPM_optz_process_2} suggest that one may consider using QPM instead of ALM in practice as it converges in fewer number of iterations.

We further investigate whether ALM and QPM yield similar structure learning performance. The results with ER1 graphs are reported in Figure \ref{fig:ALM_vs_QPM_results} in Appendix \ref{sec:supp_exp_results}, showing that ALM performs similarly to QPM across all metrics. All these observations appear to generalize to the case with $\ell_1$ regularization term that is not covered by our analysis in Section \ref{sec:optimization}.

{\bf Real data.} \ \ \ 
We conduct empirical studies to verify whether our observations hold on the protein signaling dataset by \citet{Sachs2005causal}. Due to the limited space, the optimization processes of ALM and QPM on this dataset are shown in Figure \ref{fig:ALM_vs_QPM_optz_process_sachs}, which yield consistent observations with those on synthetic data, i.e., ALM behaves similarly to QPM that requires the penalty coefficient to be very large (e.g., $10^{12}$), both of which converge to a solution whose value of $h_\text{exp}(B(\theta_k))$ is very close to zero, yielding a DAG after thresholding.

\subsection{Different Optimization Algorithms for Handling Ill-Conditioning}\label{sec:exp_optimizer}
The previous experiments demonstrate that ALM behaves similarly to QPM that requires bringing the penalty coefficient to infinity in order to converge to a DAG solution, which is known to cause numerical difficulties and ill-conditioning issues on the objective landscape \citep{Bertsekas1999nonlinear,Nocedal2006numerical}. Here we experiment with different optimization algorithms for solving the QPM subproblems of NOTEARS and NOTEARS-MLP-L1, to investigate which of them handle ill-conditioning better and produce a better solution in practice. The optimization algorithms include gradient descent with momentum \citep{Qian1999momentum}, Nesterov accelerated gradient (NAG) \citep{Nesterov1983method}, Adam \citep{Kingma2014adam} and L-BFGS \citep{Byrd2003lbfgs}; see Appendix \ref{sec:supp_exp_details} for the implementation details. Note that gradient descent with momentum and NAG often terminate earlier because of numerical difficulties, so we report its results right before termination. 

Due to the space limit, the optimization processes of NOTEARS and NOTEARS-MLP-L1 on the $100$-node ER1 graphs are visualized in Figure \ref{fig:optimizer_optz_process} in Appendix \ref{sec:supp_exp_results}. One first observes that momentum and NAG terminate when the coefficient $\rho_k$ reaches $10^7$ and $10^6$, respectively, in the linear case, and reaches $10^9$ and $10^8$, respectively, in the nonlinear case, indicating that they fail to handle ill-conditioning in these settings. In the linear case, L-BFGS is more stable than Adam for large $\rho_k$, thus returning a solution with a lower objective value $f(\theta_k)$ and SHD. The opposite is observed in the nonlinear case where Adam has consistently lower $f(\theta_k)$ and SHD than L-BFGS. Similar observations are also made for the overall structure learning performance on ER1 and SF4 graphs, as depicted in  Figures \ref{fig:optimizer_results_1} and \ref{fig:optimizer_results_2} in Appendix \ref{sec:supp_exp_results}, with graph sizes $d\in\{10, 20, 50, 100\}$. In the linear case, L-BFGS performs the best across most settings, while the performance of Adam is slightly better than L-BFGS in the nonlinear case. Gradient descent with momentum and NAG give rise to much higher SHD and SID, especially on large graphs.

\begin{figure}[!t]
\centering
\includegraphics[width=0.36\textwidth]{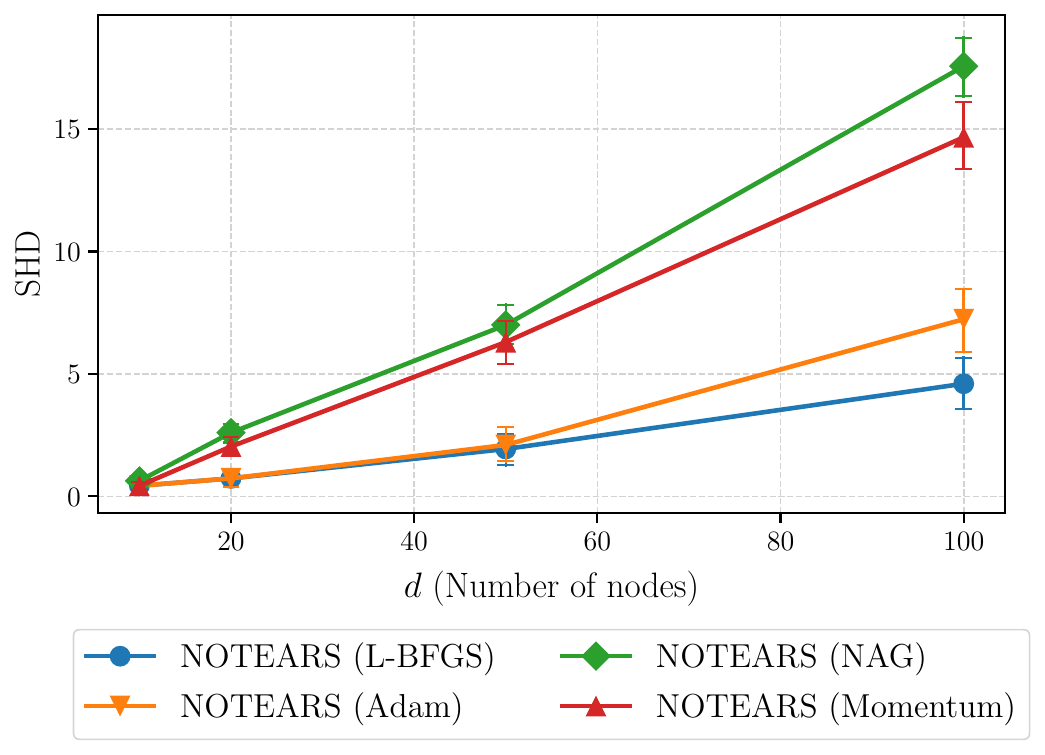}
\caption{Empirical results of different optimization algorithms for solving the QPM subproblems of NOTEARS on synthetic data. The ground truths are ER1 graphs and the sample size is $n=1000$.}
\label{fig:optimizer_results_1}
\end{figure}

As compared to first-order method, the observations above suggest that second-order method such as L-BFGS handles ill-conditioning better by incorporating curvature information through approximations of the Hessian matrix, which verifies our findings in Section \ref{sec:resolve_convergence}. This is also consistent with the optimization literature \citep{Bertsekas1999nonlinear} and the bivariate example in Appendix \ref{sec:bivariate_example}. The Adam algorithm, on the other hand, lies in the middle as it employs diagonal rescaling on the parameter space by maintaining running averages of past gradients \citep{Bottou2018optimization}. In the nonlinear case, it may be surprising that Adam performs slightly better than L-BFGS, as its estimate of the Hessian matrix is not as accurate as that of L-BFGS. Nevertheless, this demonstrates its effectiveness for training MLPs, and may be understandable given its popularity in various deep learning tasks \citep{Schmidt2021descending}.

\subsection{Further Resolving the Convergence Issue}\label{sec:exp_resolve_convergence}
In this section, we provide empirical comparisons of different methods for resolving the convergence issue of the NOTEARS formulation, as described in Section \ref{sec:resolve_convergence}. In particular, we compare NOTEARS-L1 (with L-BFGS) to Abs-KKTS\footnote{We consider only Abs-KKTS instead of NOTEARS-KKTS because otherwise we also have to apply the KKT-informed local search for NoCurl and GOLEM-EV to ensure a fair comparison, which is not the focus of our work.}, NoCurl, and GOLEM-EV. The implementation details of these methods are described in Appendix \ref{sec:supp_exp_details}. Here we consider ER1 and SF4 graphs with $1000$ and $3d$ samples.

Figure \ref{fig:with_baselines_1000_samples_1} shows the SHD on ER1 graphs with $1000$ samples, while the complete results on ER1 and SF4 graphs with $1000$ and $3d$ samples can be found in Figures \ref{fig:with_baselines_1000_samples} and \ref{fig:with_baselines_3d_samples} in Appendix \ref{sec:supp_exp_results}. Overall, GOLEM-EV has the best performance across nearly all settings, especially in terms of SID. NOTEARS-L1 and Abs-KKTS perform similarly on SF4 graphs, while the former has lower SHD and SID on ER1 graphs. NoCurl has a high SHD especially for the case with $3d$ samples, which indicates that it requires a larger sample size to perform well. This experiment suggests that despite being susceptible to ill-conditioning, NOTEARS-L1 with L-BFGS is still very competitive in practice and performs even better than Abs-KKTS and NoCurl that are less susceptible to the convergence issue, possibly because L-BFGS can help remedy the ill-conditioning issue, as demonstrated in Section \ref{sec:exp_optimizer}.

\begin{figure}[!t]
\centering
\includegraphics[width=0.36\textwidth]{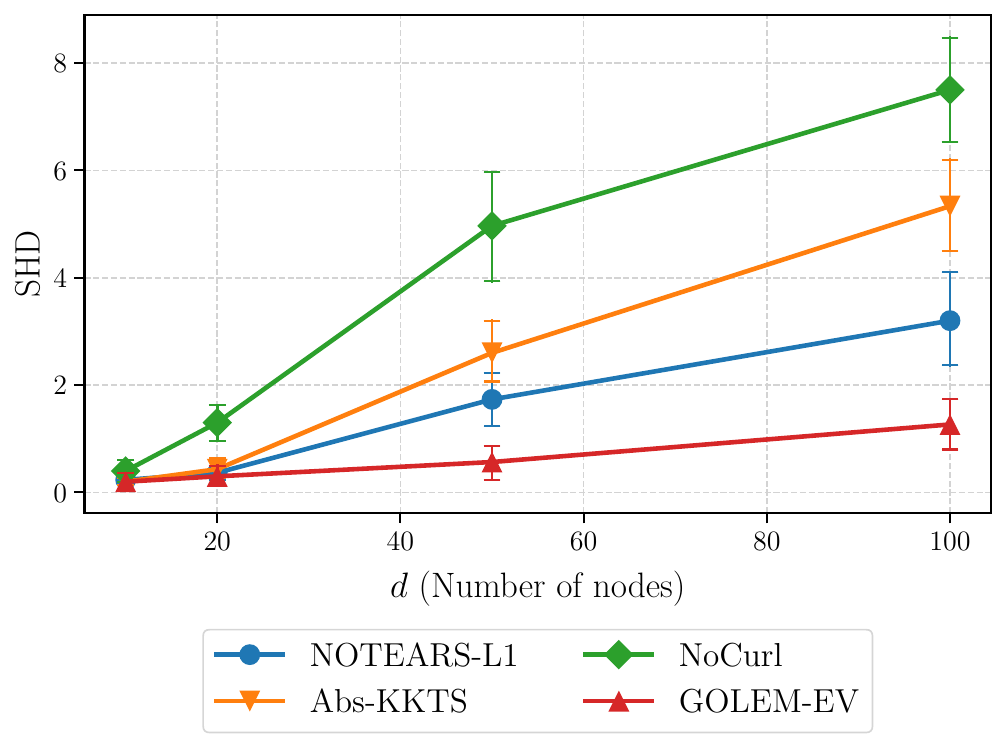}
\caption{Empirical results of different methods on synthetic data with ER1 graphs and $n=1000$ samples.}
\label{fig:with_baselines_1000_samples_1}
\end{figure}

%% file: sections/5conclusion.tex
\section{Conclusion}\label{sec:conclusion}
We examined the convergence of ALM and QPM for structure learning in the linear, nonlinear, and confounded cases. In particular, we dug into the standard convergence result of ALM and showed that the required regularity conditions are not satisfied in this setting. Further experiments demonstrate that ALM behaves similarly to QPM that requires bringing the penalty coefficient to infinity and is prone to ill-conditioning. We then showed theoretically and empirically that QPM guarantees convergence to a DAG solution, under mild conditions. The empirical studies also suggest that our analysis generalizes to the cases with $\ell_1$ regularization. Lastly, we connected our theoretical results with different approaches to help resolve the convergence issue, and provided an empirical comparison of them to further illustrate our findings. In particular, it is worth noting that second-order method can help remedy the ill-conditioning issue of the continuous constrained formulation and therefore leads to competitive structure learning results in practice.

%% file: sections/6acknowledgment.tex
\section*{Acknowledgments}
The authors would like to thank the anonymous reviewers for their useful comments. This work was supported in part by the National Institutes of Health (NIH) under Contract R01HL159805, by the NSF-Convergence Accelerator Track-D award \#2134901, by the United States Air Force under Contract No. FA8650-17-C7715, by a grant from Apple, by the Canada CIFAR AI Chair Program, by an IVADO excellence PhD scholarship, and by a Google Focused Research award. The NIH or NSF is not responsible for the views reported in this article. Simon Lacoste-Julien is a CIFAR Associate Fellow in the Learning in Machines \& Brains program.

%% file: sections/7appendix.tex
\section{Optimality Conditions for Equality Constrained Problems}\label{sec:optimality_condition}
We review the optimality conditions for equality constrained optimization problems, which are required for the study in Sections \ref{sec:alm_background} and \ref{sec:regularity}. Note that the following conditions are adopted from Theorems $12.1$ and $12.6$ in \citet{Nocedal2006numerical}, respectively.
\begin{definition}[First-order necessary conditions]\label{def:first_order}
Suppose that $\theta^*$ is a local solution of \eqref{eq:generic-opt}, that the functions $f$ and $h$ in \eqref{eq:generic-opt} are continuously differentiable, and that the LICQ holds at $\theta^*$. Define the Lagrangian function of \eqref{eq:generic-opt} as
\[
\mathcal{L}(\theta, \alpha) = f(\theta) + \alpha^\T h(\theta),
\]
where $\alpha \in \mathbb{R}^p$. Then there is a Lagrange multiplier vector $\alpha^*$ such that the following conditions are satisfied at $(\theta^*, \alpha^*)$:
\begin{subequations}\label{eq:first_order}
\begin{align}
\nabla_\theta \mathcal{L}(\theta^*, \alpha^*)&=0,\\
h(\theta^*)&=0.
\end{align}
\end{subequations}
\end{definition}

\begin{definition}[Second-order sufficient conditions]\label{def:second_order}
Let $\mathcal{L}(\theta, \alpha)$ be the Lagrangian function of \eqref{eq:generic-opt} as in Definition \ref{def:first_order}. Suppose that for some feasible point $\theta^*$ there is a Lagrange multiplier vector $\alpha^*$ such that the conditions \eqref{eq:first_order} are satisfied. Suppose also that
\[
y^\T \nabla_{\theta\theta}^2 \mathcal{L}(\theta^*, \alpha^*) y > 0,\ \ \text{for all~} y\neq 0 \text{~with~} \nabla_\theta h(\theta^*) y=0.
\]
Then $\theta^*$ is a strict local solution for \eqref{eq:generic-opt}.
\end{definition}

\section{Proofs}

\subsection{Proof of Theorem \ref{thm:constraint_notears}}\label{sec:proof_constraint_notears}

\subsubsection*{Proof for function $h_\textnormal{exp}(W)$.}
We first provide the proof for function $h_\text{exp}(W)$. Its gradient is given by
\[
\nabla_W h_\text{exp}(W) = (e^{W\odot W})^\T \odot 2W.
\]

\textbf{If part:} \\
We first consider the diagonal entries $W_{ii}$ for $i\in[d]$\footnote{We denote by $[d]$ the set $\{1,2,\dots,d\}$, likewise for the others.}. Since $\left((e^{W\odot W})^\T\right)_{ii}\geq 1$,  $(\nabla_W h_\text{exp}(W))_{ii}=0$ implies that $W_{ii}=0$, or equivalently, the structure defined by $W$ does not have any self-loop. Now we consider the $(j, i)$-th entry of $W$ with $j \neq i$. $(\nabla_W h_\text{exp}(W))_{ji}=0$ indicates that at least one of $W_{ji}$ and $\left((e^{W\odot W})^\T\right)_{ji}$ is zero. Therefore, the edge from node $j$ to node $i$, if exists, must not belong to any cycle. Combining the above cases, we conclude that all edges must not be part of any self-loop or cycle, and that $W$ represents a DAG, i.e., $h_\text{exp}(W)=0$.

\textbf{Only if part:} \\
Notice that $h_\text{exp}(W) = 0$ implies that $W$ represents a DAG. Since there is not any self-loop, the diagonal entries of $W$ are zero, and so are the diagonal entries of $\nabla_W h_\text{exp}(W)$. It remains to consider the $(j, i)$-th entry of $\nabla_W h_\text{exp}(W)$ with $j \neq i$:
\begin{itemize}
  \item If $W_{ji} = 0$, then it is clear that $(\nabla_W h_\text{exp}(W))_{ji}=0$.
  \item If $W_{ji} \neq 0$, then there is an edge from node $j$ to node $i$ with weight $W_{ji}$. The other term $\left((e^{W\odot W})^\T\right)_{ji}$ indicates the total number of weighted walks from node $i$ to node $j$. If both $W_{ji}$ and $\left((e^{W\odot W})^\T\right)_{ji}$ are nonzero, then there is at least a weighted closed walk passing through nodes $i$ and $j$, contradicting the statement that $W$ represents a DAG.
\end{itemize}
Therefore, at least one of $W_{ji}$ and $\left((e^{W\odot W})^\T\right)_{ji}$ must be zero, and we have $(\nabla_W h_\text{exp}(W))_{ji}=0$.

\subsubsection*{Proof for function $h_\textnormal{exp}(B(A))$.}
We now provide the proof for function $h_\text{exp}(B(A))$, which is an extension to the proof for function $h_\text{exp}(W)$. Let $A^{(t)}=(A_1^{(t)},\dots,A_d^{(t)})\in\mathbb{R}^{d\times s_{t-1}\times s_{t}}$ denote the weights in the $t$-th layer of the MLPs corresponding to all variables. Since $B(A)$ depends only on $A^{(1)}$ by definition, we have
\[
\nabla_{A^{(t)}} h_\text{exp}(B(A)) = 0, \quad t=2,\dots,\ell.
\]
Therefore, it suffices to consider the case of $t=1$, i.e., the weights in the first layer of the MLPs , and show that $h(B(A))=0$ if and only if $\nabla_{A^{(1)}} h_\text{exp}(B(A))=0$. Note that we have $A^{(1)}\in\mathbb{R}^{d\times d\times s_1}$, and the $(j,i,s)$-entry of the gradient is given by
\[
(\nabla_{A^{(1)}} h_\text{exp}(B(A)))_{jis} = ((e^{B(A)\odot B(A)})^\T)_{ji} \odot 2(A^{(1)})_{jis}, \quad j,i\in [d], \ s\in [s_1].
\]
For clarity, we restate the definition of $B(A)$ in terms of $A^{(1)}$:
\begin{flalign}\label{eq:B_mlp}
(B(A))_{ji}&=\| j\text{th-row}(A_{i}^{(1)} ) \|_2\nonumber\\
&=\left(\sum_{s=1}^{s_1}((A^{(1)})_{jis})^2\right)^{\frac{1}{2}},\quad j,i\in [d].
\end{flalign}

\textbf{If part:} \\
We first consider the diagonal entries $(B(A))_{ii}$ for $i\in[d]$. Since $\left((e^{B(A) \odot B(A)})^\T\right)_{ii}\geq 1$,  $(\nabla_{A^{(1)}} h_\text{exp}(B(A)))_{iis}=0$ implies that $(A^{(1)})_{iis}=0$ for $s\in [s_1]$ and therefore $(B(A))_{ii}=0$, indicating that the structure defined by $B(A)$ does not have any self-loop. Now we consider the $(j, i)$-th entry of $B(A)$ with $j \neq i$. Suppose $(B(A))_{ji}\neq 0$, i.e., there is an edge from node $j$ to node $i$ in the structure defined by $B(A)$. By Eq. \eqref{eq:B_mlp}, there exists $s\in [s_1]$ such that $(A^{(1)})_{jis}>0$, which, with the assumption of $(\nabla_{A^{(1)}} h_\text{exp}(B(A)))_{jis}=0$, implies that $\left((e^{B(A)\odot B(A)})^\T\right)_{ji}=0$. This indicates that there is not any weighted walk from node $i$ to node $j$. Therefore, the edge from node $j$ to node $i$, if exists, must not belong to any cycle. Combining the above cases, we conclude that all edges must not be part of any self-loop or cycle, and that $B(A)$ represents a DAG, i.e., $h_\text{exp}(B(A))=0$.

\textbf{Only if part:} \\
Notice that $h_\text{exp}(B(A)) = 0$ implies that $B(A)$ represents a DAG. Since there is not any self-loop, $(B(A))_{ii}=0$ indicates that $(A^{(1)})_{iis}=0$ for $s\in[s_1]$, and therefore $(\nabla_{A^{(1)}} h_\text{exp}(B(A)))_{iis}=0$. It remains to consider the $(j, i, s)$-th entry of $\nabla_{A^{(1)}} h_\text{exp}(B(A))$ with $j \neq i$ and $s\in[s_1]$:
\begin{itemize}
  \item If $(A^{(1)})_{jis}=0$, then it is clear that $(\nabla_{A^{(1)}} h_\text{exp}(B(A)))_{jis}=0$.
  \item If $(A^{(1)})_{jis} \neq 0$, then we must have $(B(A))_{ji}\neq 0$ by Eq. \eqref{eq:B_mlp} and there is an edge from node $j$ to node $i$ in the structure defined by $B(A)$. The other term $\left((e^{B(A) \odot B(A)})^\T\right)_{ji}$ indicates the total number of weighted walks from node $i$ to node $j$. If both $(A^{(1)})_{jis}$ and $\left((e^{B(A)\odot B(A)})^\T\right)_{ji}$ are nonzero, there is at least a weighted closed walk passing through nodes $i$ and $j$, contradicting the statement that $B(A)$ represents a DAG.
\end{itemize}
Therefore, at least one of $(A^{(1)})_{jis}$ and $\left((e^{B(A)\odot B(A)})^\T\right)_{ji}$ must be zero, and we have $(\nabla_{A^{(1)}} h_\text{exp}(B(A)))_{jis}=0$.

\subsection{Proof of Theorem \ref{thm:qpm}}\label{sec:proof_qpm_convergence}
By differentiating $Q(\theta; \rho_k)$ in Eq. \eqref{eq:qpm}, we obtain
\[
\nabla_\theta Q(\theta_k; \rho_k) = \nabla_\theta f(\theta_k) + \rho_k h(B(\theta_k))\nabla_\theta h(B(\theta_k)) .
\]
From the termination criterion of the subproblems in Algorithm \ref{alg:qpm}, we have
\[
\left \| \nabla_\theta f(\theta_k) + \rho_k h(B(\theta_k))\nabla_\theta h(B(\theta_k)) \right \|_2 \leq \tau_k \leq \tau,\]
where $\tau$ is an upper bound on $\{\tau_k\}$. By rearranging this expression (and in particular using the inequality $\|a\|_2 - \|b\|_2\leq \|a+b\|_2$), we obtain
\[
\|h(B(\theta_k))\nabla_\theta h(B(\theta_k))\|_2 \leq \frac{1}{\rho_k}\left(\tau + \|\nabla_\theta f(\theta_k)\|_2\right).
\]
Let $\theta^*$ be a limit point of the sequence of iterates. Then there is a subsequence $\mathcal{K}$ such that $\lim_{k\in\mathcal{K}} \theta_k = \theta^*$. By the continuity of $\nabla_\theta f$, when we take limits as ${k\rightarrow \infty}$ for $k\in\mathcal{K}$, the bracketed term on the right-hand-side approaches $\tau + \|\nabla_\theta f(\theta^*)\|_2$, so because $\rho_k \rightarrow \infty$, the right-hand-side approaches zero. This implies that $\lim_{k\in\mathcal{K}}\|h(B(\theta_k))\nabla_\theta h(B(\theta_k))\|_2 = 0$. By the continuity of $h$, $\nabla_\theta h$, and $B$, we conclude that
\[
h(B(\theta^*))\nabla_\theta h(B(\theta^*)) = 0, 
\]
which, by Assumption \ref{assump:dag_constraint}, yields
\[
h(B(\theta^*)) = 0.
\]

\subsection{Proof of Theorem \ref{thm:constraint_bowfree}}\label{sec:proof_constraint_bowfree}
The gradient of $h_\text{bf}(W, \Omega)$ is given by
\begin{align*}
    \nabla_W h_\text{bf}(W, \Omega) &= (e^{W\odot W})^\T \odot 2W + 2W \odot \Omega \odot \Omega,\\
    \nabla_\Omega h_\text{bf}(W, \Omega) &= 2 W\odot W\odot\Omega.
\end{align*}

\textbf{If part:} \\
The gradient term $\nabla_\Omega h_\text{bf}(W, \Omega) = 2 W\odot W\odot\Omega=0$ indicates that $W_{ji}\Omega_{ji}=0$ for $j,i\in[d]$, and thus there is not any bow in the structure defined by $W$ and $\Omega$. Therefore, we also have $2W \odot \Omega \odot \Omega=0$, which, with $\nabla_W h_\text{bf}(W, \Omega)=0$, implies that $(e^{W\odot W})^\T \odot 2W=0$. By Theorem \ref{thm:constraint_notears}, this indicates that there is not any directed cycle. Therefore, the structure defined by $W$ and $\Omega$ is a bow-free ADMG and we have $h_\text{bf}(W, \Omega)=0$.

\textbf{Only if part:} \\
Notice that $h_\text{bf}(W, \Omega) = 0$ implies that the structure defined by $W$ and $\Omega$ is a bow-free ADMG. Since there is not any directed cycle in an ADMG, we have $(e^{W\odot W})^\T \odot 2W=0$ by Theorem \ref{thm:constraint_notears}. Furthermore, since there is not any bow in the structure, we must have $W_{ji}\Omega_{ji}=0$ for $j,i\in[d]$, which implies $2W \odot \Omega \odot \Omega=0$ and $2W\odot W\odot\Omega=0$. Therefore, we have $\nabla_W h_\text{bf}(W, \Omega)=0$ and $\nabla_\Omega h_\text{bf}(W, \Omega)=0$.

\section{Bivariate Example: Ill-Conditioning and Optimization Algorithms}\label{sec:bivariate_example}
We provide a bivariate example to illustrate the ill-conditioning issue and the effectiveness of different optimization algorithms. Here we focus on the asymptotic case in which the true covariance matrix is known.

{\bf Setup.} \ \ \ 
Consider the bivariate linear Gaussian model defined by the weighted adjacency matrix
\[W_0 = \begin{bmatrix}
0& 2 \\ 
0 & 0
\end{bmatrix}\]
and noise covariance matrix
\[\Omega_0 = \begin{bmatrix}
1& 0 \\ 
0 & 1
\end{bmatrix}.\]
In other words, the additive noises are assumed to be standard Gaussians. The true covariance matrix of $X$ is then given by
\begin{flalign*}
\Sigma_0 &= (I-W_0)^{-\T} \Omega_0 (I-W_0)^{-1}\\
&=\begin{bmatrix}
1 & 2 \\ 
2 & 5
\end{bmatrix}.
\end{flalign*}
We define the matrix $W$ with variables $b$ and $c$ as
\[W = \begin{bmatrix}
0& b \\ 
c & 0
\end{bmatrix},\]
which yields the corresponding least squares objective
\begin{flalign*}
f(W) &= \frac{1}{2}\Tr\big( (I-W)^{\T} \Sigma_0 (I-W) \big)\\
&= \frac{1}{2}\left( (b - 2)^2 + (2 c - 1)^2 + c^2 + 1 \right).
\end{flalign*}
We consider a simplified DAG constraint $bc=0$, which leads to the quadratic penalty function
\[
Q(W;\rho)=f(W)+\frac{\rho}{2}b^2c^2.
\]
Clearly, $W$ represents a DAG if and only if $bc=0$.
\begin{figure}[!t]
\centering
\subfloat[$\rho=0$.]{
  \includegraphics[width=0.313\textwidth]{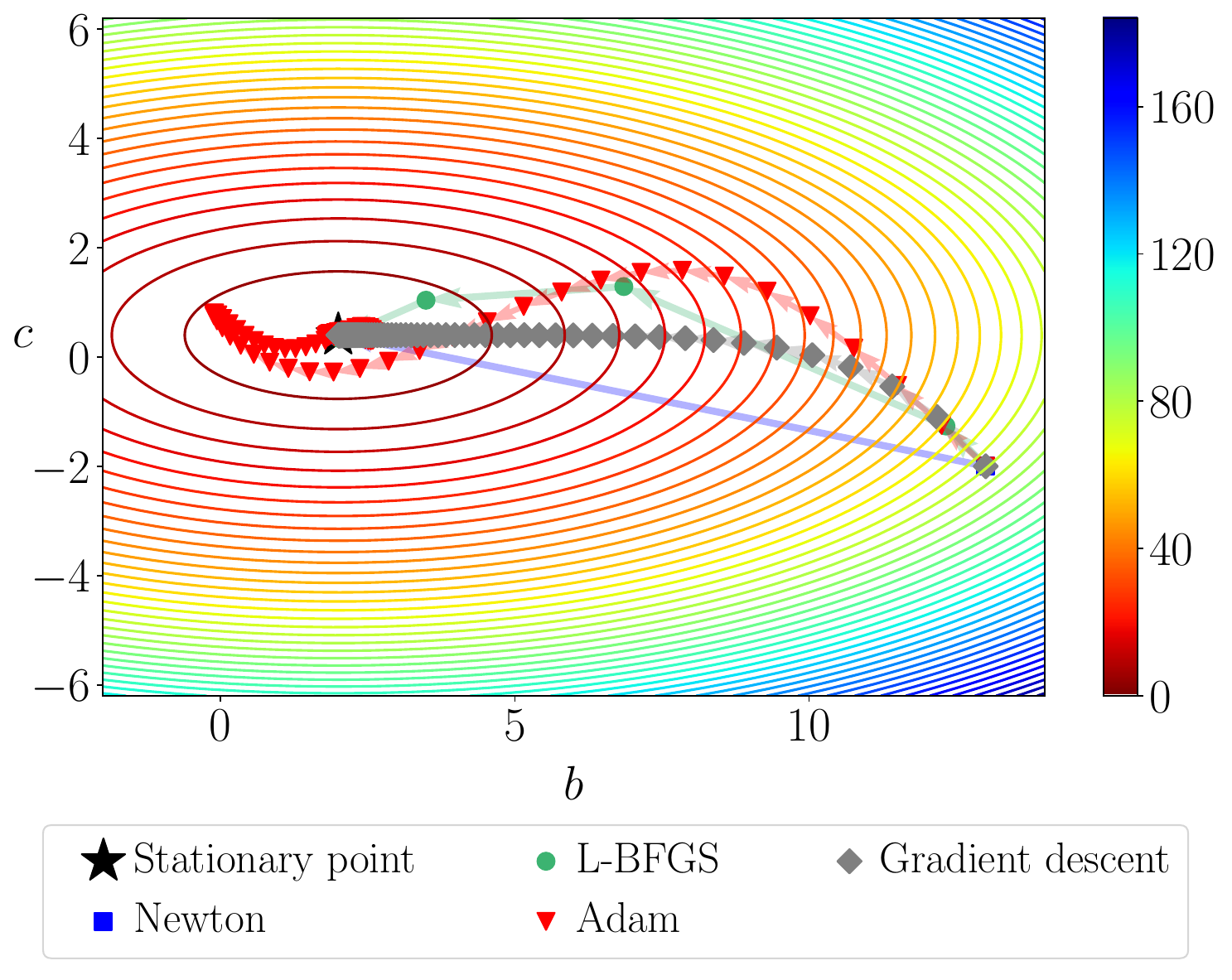}
}
\subfloat[$\rho=1$.]{
  \includegraphics[width=0.33\textwidth]{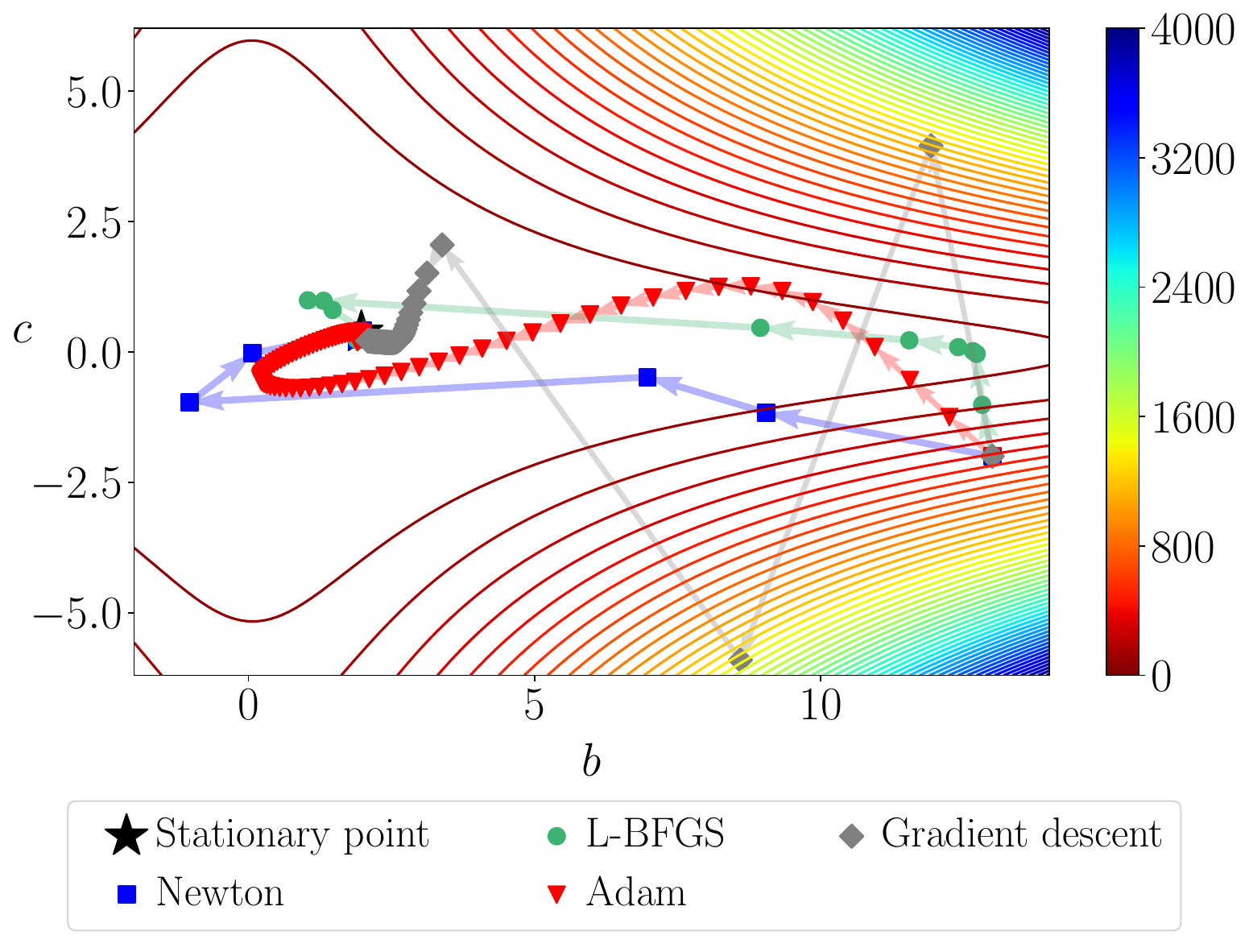}
}
\subfloat[$\rho=10^6$.]{
  \includegraphics[width=0.333\textwidth]{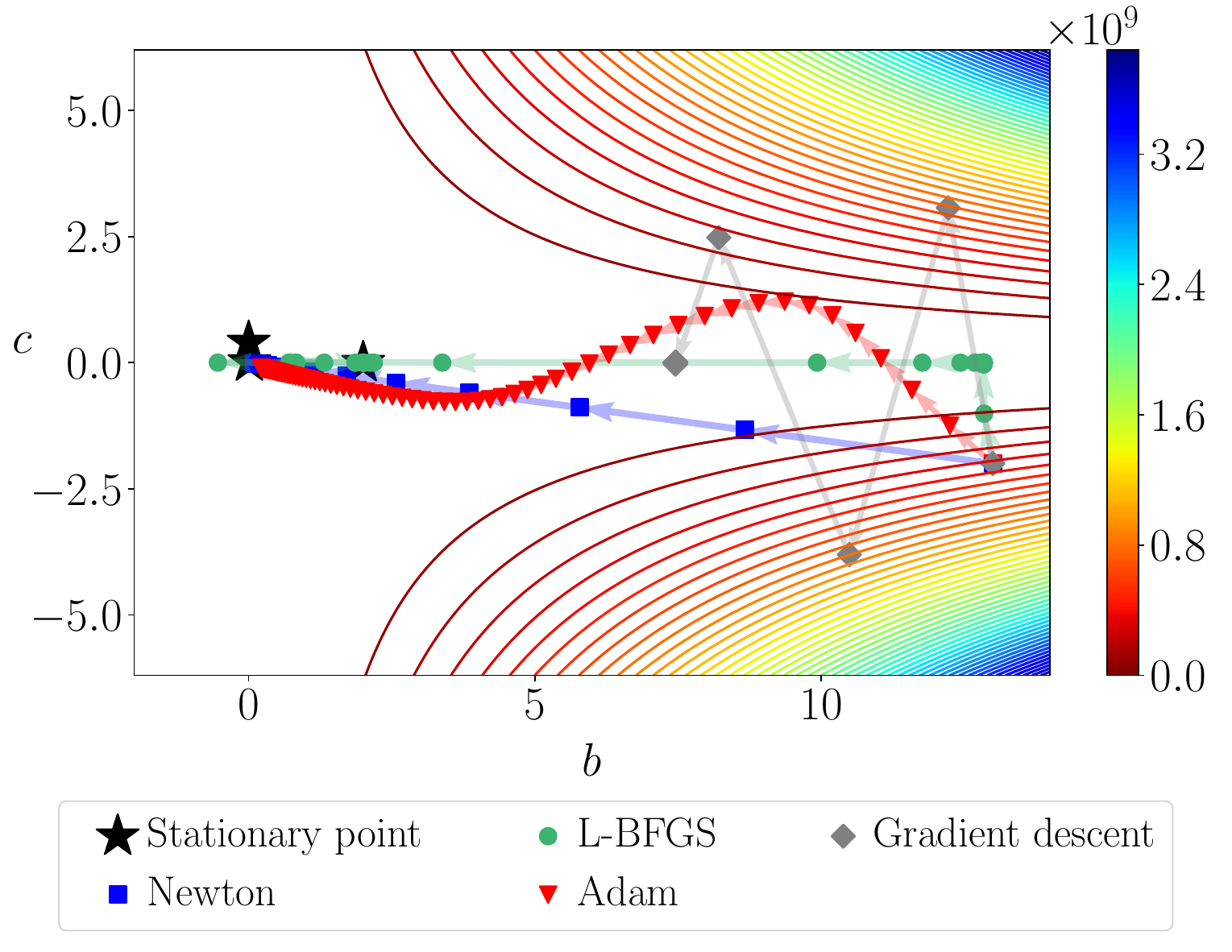}
}
\caption{Contour plots of the quadratic penalty function $Q(W;\rho)$ with different penalty coefficients. The optimization trajectories of different algorithms for minimizing $Q(W;\rho)$ are visualized.}
\label{fig:bivariate_example}
\end{figure}

{\bf Contour.} \ \ \ 
The contour plots of $Q(W;\rho)$ w.r.t. variables $b$ and $c$ with different penalty coefficients, i.e., $\rho=0,1,10^6$, are visualized in Figure \ref{fig:bivariate_example}. When $\rho=0$, the contour is elliptical as the least squares objective corresponds to a quadratic function. With a small penalty coefficient $\rho=1$, the contour is slightly stretched out and deviates from being elliptical. When a large coefficient $\rho=10^6$ is used, the function contour is highly stretched out along both x and y axes. In this case, when the values of $b$ and $c$ are moderately large, a small change in either of them can lead to a large change in the penalty function $Q(W;\rho)$.

{\bf Optimization algorithms.} \ \ \ 
We further study the behavior of different optimization algorithms for solving the minimization problem of $Q(W;\rho)$. In particular, we visualize the trajectories of Newton's method, L-BFGS \citep{Byrd2003lbfgs}, Adam \citep{Kingma2014adam}, and gradient descent starting from the point $(b,c)=(13, -2)$. We pick an initial point that is relatively far from the origin for better illustration of the ill-conditioning issue.

The optimization trajectories and the stationary points for different penalty coefficients are visualized in Figure \ref{fig:bivariate_example}. When $\rho=0,1$, all four algorithms converge to the unique stationary point. With a small penalty coefficient $\rho=1$, one observes that gradient descent moves along a zigzag path during the first few iterations. A possible reason is that the gradients may not be the best direction to descend to the stationary point. This phenomenon becomes more severe when $\rho=10^6$ is used; specifically, gradient descent follows a zigzag path and eventually reaches a point where it no longer makes progress, owing to the ill-posed optimization landscape. In this case,  Newton's method and Adam converge to the same stationary point, while L-BFGS converges to the other one. It is also observed that Newton's method converges in the fewest number of iterations in all cases, and L-BFGS is on par with it as compared to Adam and gradient descent.

The observations above demonstrate that ill-conditioning is a huge concern especially for first-order method like gradient descent, which may produce a bad solution or not converge well. They also corroborate our findings in Section \ref{sec:exp_optimizer}, i.e, second-order method like Newton's method and L-BFGS helps resolve the ill-conditioning issue by incorporating the curvature information. In particular, Newton's method computes the direction of descent using both first-order (Jacobian) and second-order (Hessian) information, while, instead of explicitly evaluating the Hessian matrix, L-BFGS relies on the approximations of Hessian. Therefore, Newton's method is computationally less efficient in practice, which thus is not included in the experiments in Section \ref{sec:exp_optimizer}. On the other hand, the Adam algorithm, loosely speaking, lies in the middle between first-order and second-order methods as it employs diagonal rescaling on the parameter space that can be interpreted as second-order-type information \citep{Bottou2018optimization}. Therefore, it can be less susceptible to ill-conditioning as compared to gradient descent, but may require more iterations to converge than Newton's method and L-BFGS.

In the empirical study above, we find that gradient descent is sensitive to the choice of learning rate (or referred to as the step size). Therefore, we manually pick it such that the algorithm does not diverge or converge too slowly. For the other algorithms, they are relatively robust to the choice of learning rate. 

\section{Supplementary Experiment Details}\label{sec:supp_exp_details}
This section provides further experiment details of different optimization algorithms and structure learning methods for Section \ref{sec:exp}.

{\bf Optimization algorithms.} \ \ \ 
We implement the optimization algorithms including gradient descent with momentum, NAG, and Adam with PyTorch \citep{Paszke2019pytorch}. For Adam, we use a learning rate of $10^{-3}$. For gradient descent with momentum and NAG, we set the learning rate to $10^{-4}$ and the momentum factor to $0.9$. We set the number of optimization iterations to $10^4$ for all these algorithms. We use the implementation and default hyperparameters of L-BFGS released by \citet{Zheng2018notears,Zheng2020learning}.

{\bf Structure learning methods.} \ \ \ 
The implementations of Abs-KKTS\footnote{\url{https://github.com/skypea/DAG_No_Fear}}, NoCurl\footnote{\url{https://github.com/fishmoon1234/DAG-NoCurl}}, and GOLEM-EV\footnote{\url{https://github.com/ignavierng/golem}} are available on the authors' GitHub repositories. For all these methods, we use the default hyperparameters and the DAG constraint term proposed by \citet{Zheng2018notears}, i.e., $h_\text{exp}(W)$, although some of the authors' original implementations adopt the polynomial alternative proposed by \citet{Yu19daggnn}, i.e., $h_\text{poly}(W)$. Similar to NOTEARS, we use a pre-processing step to center the data by subtracting the mean of each variable from the samples $\mathbf{X}$.

\section{Supplementary Experiment Results}\label{sec:supp_exp_results}
This section provides further experiment results for Section \ref{sec:exp}; see Figures \ref{fig:ALM_vs_QPM_optz_process_2}, \ref{fig:ALM_vs_QPM_results}, \ref{fig:ALM_vs_QPM_optz_process_sachs}, \ref{fig:optimizer_optz_process}, \ref{fig:optimizer_results_2}, \ref{fig:with_baselines_1000_samples}, and \ref{fig:with_baselines_3d_samples}.

\begin{figure}[!t]
\centering
\subfloat[NOTEARS-L1.]{
  \includegraphics[width=0.97\textwidth]{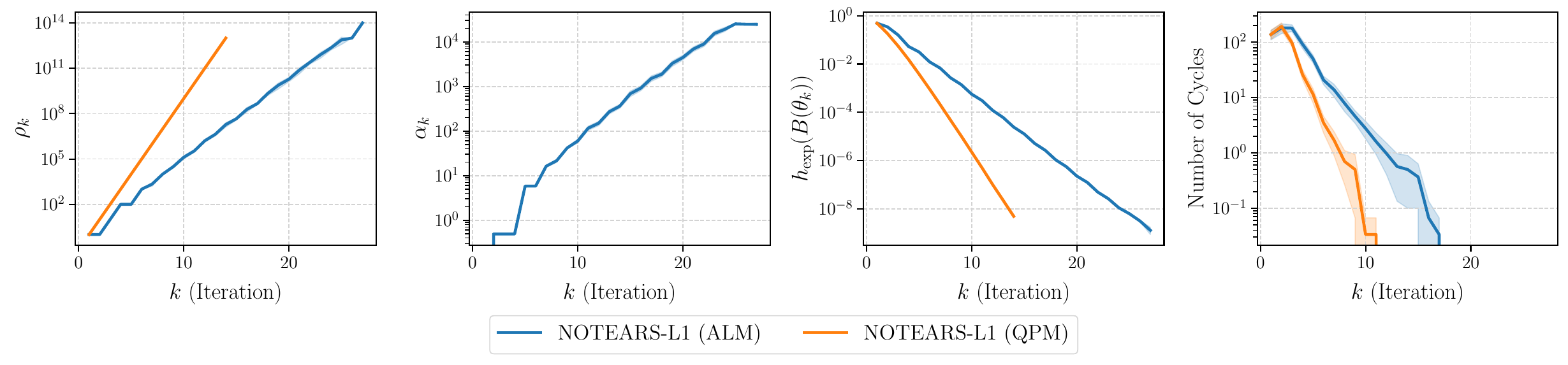}
} \\
\subfloat[NOTEARS-MLP.]{
  \includegraphics[width=0.97\textwidth]{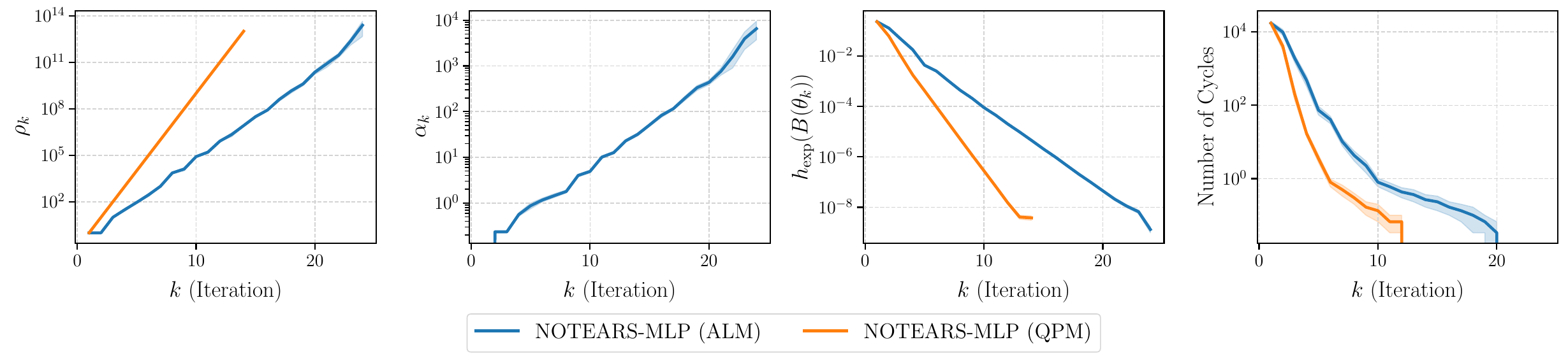}
}\\
\subfloat[NOTEARS-MLP-L1.]{
  \includegraphics[width=0.97\textwidth]{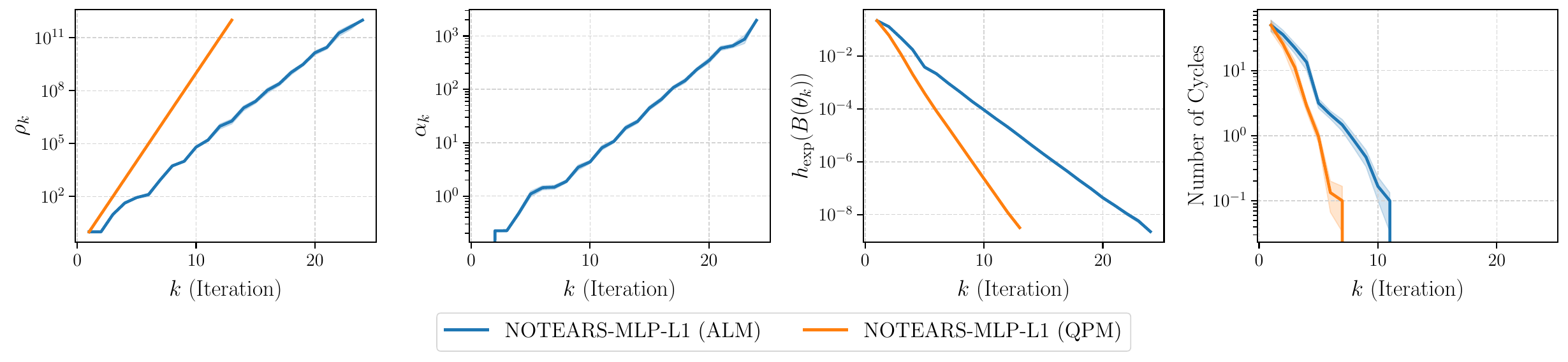}
}
\caption{Optimization processes of the continuous constrained formulation using ALM and QPM on synthetic data. The ground truths are $10$-node ER1 graphs and the sample size is $n=1000$. Each data point corresponds to the $k$-th iteration. Shaded area denotes standard errors over $30$ trials.}
\label{fig:ALM_vs_QPM_optz_process_2}
\end{figure}

\begin{figure}[!t]
\centering
\subfloat[NOTEARS.]{
  \includegraphics[width=0.97\textwidth]{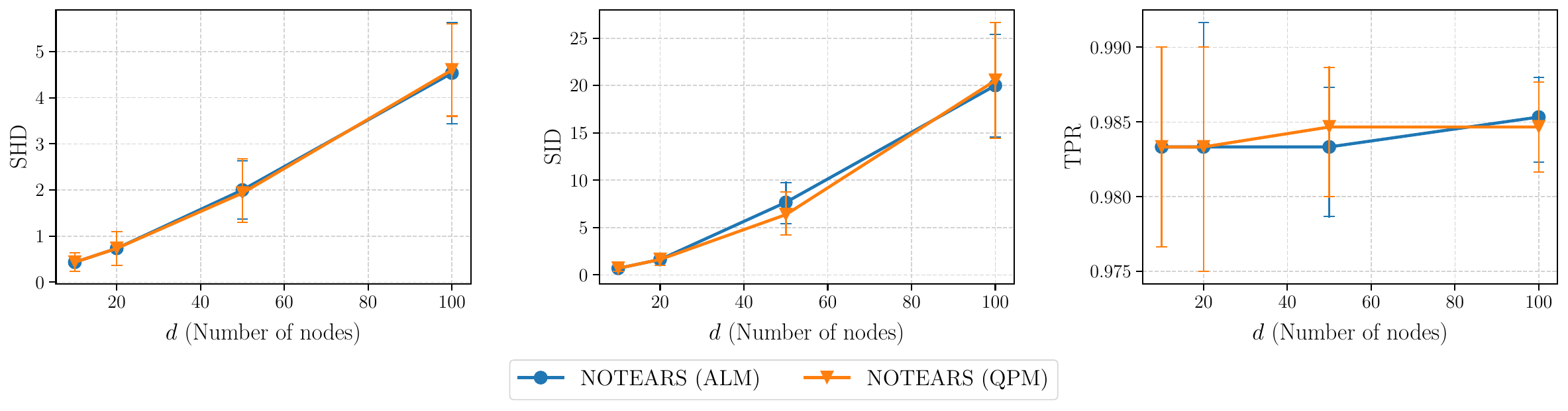}
} \\
\subfloat[NOTEARS-L1.]{
  \includegraphics[width=0.97\textwidth]{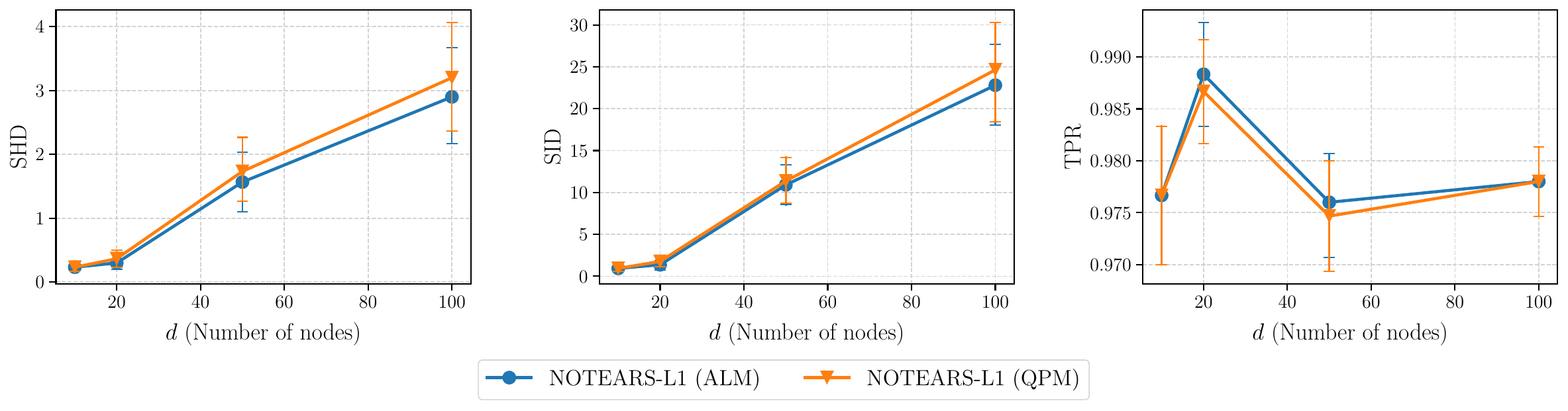}
} \\
\subfloat[NOTEARS-MLP.]{
  \includegraphics[width=0.97\textwidth]{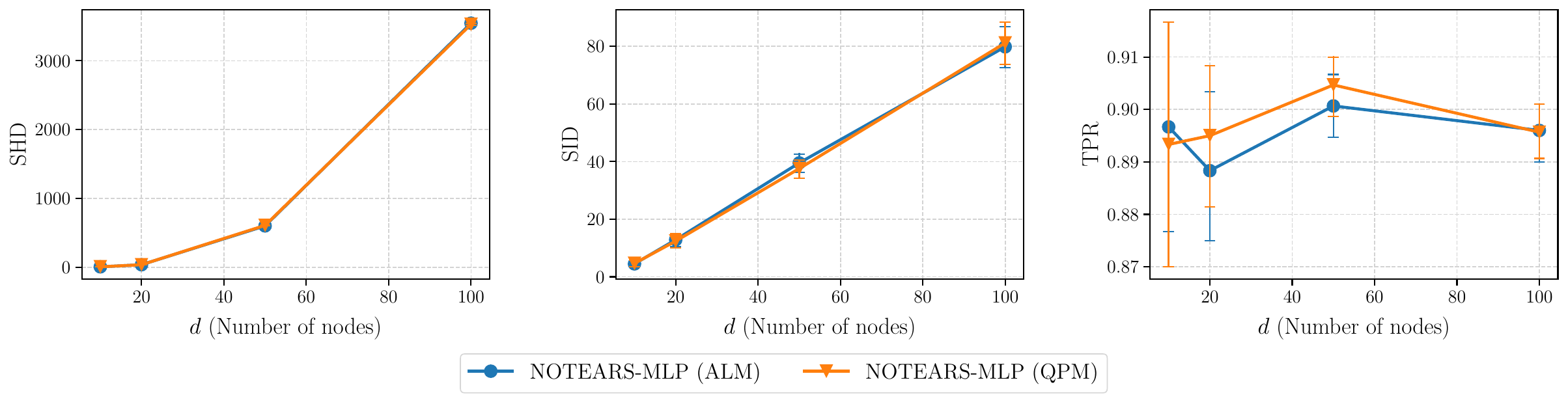}
} \\
\subfloat[NOTEARS-MLP-L1.]{
  \includegraphics[width=0.97\textwidth]{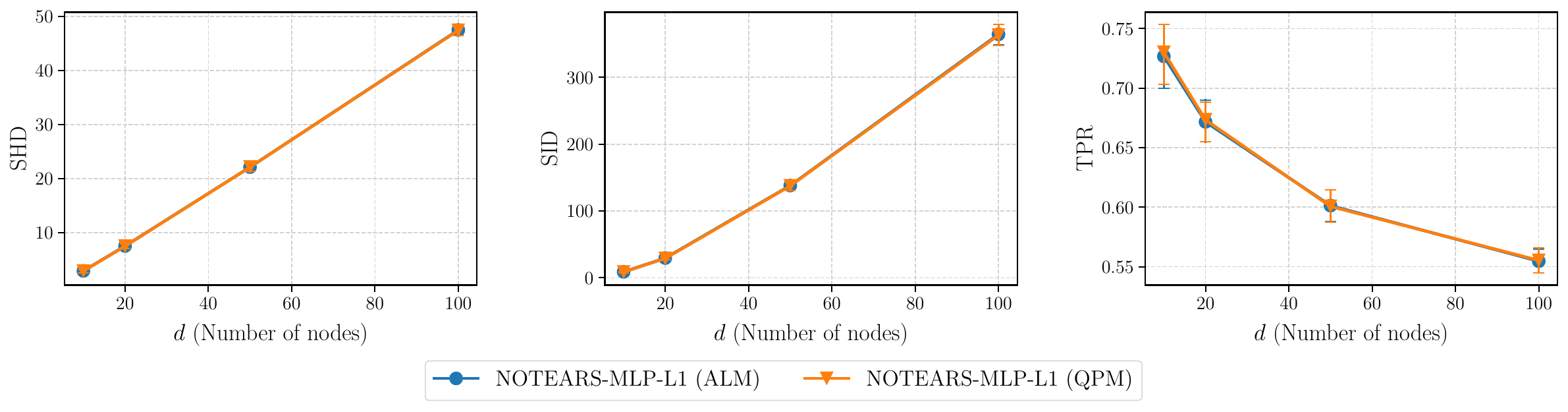}
}
\caption{Empirical results of the continuous constrained formulation using ALM and QPM on synthetic data. The ground truths are ER1 graphs and the sample size is $n=1000$. Lower is better, except for TPR. Error bars denote standard errors over $30$ trials.}
\label{fig:ALM_vs_QPM_results}
\end{figure}

\begin{figure}[!t]
\centering
\subfloat[NOTEARS.]{
  \includegraphics[width=0.97\textwidth]{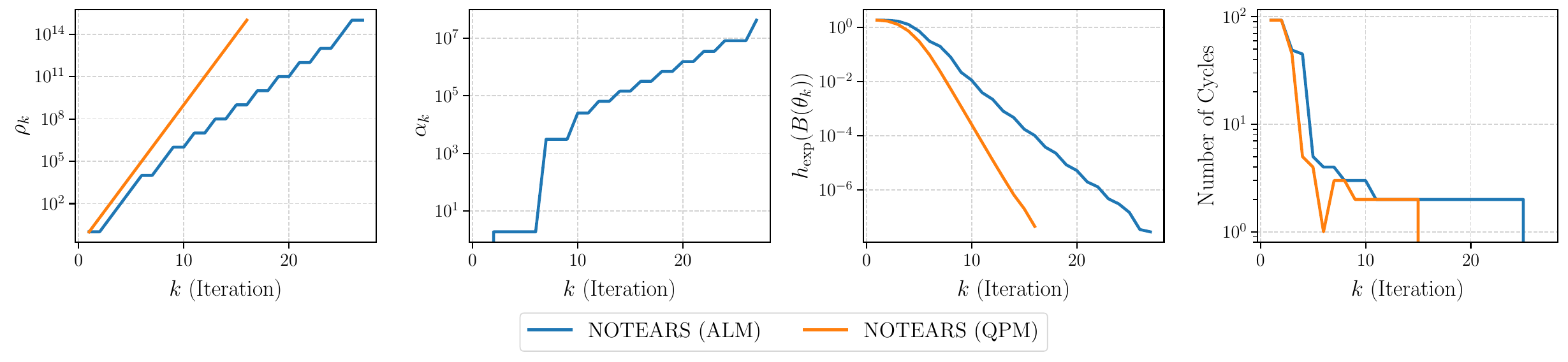}
} \\
\subfloat[NOTEARS-L1.]{
  \includegraphics[width=0.97\textwidth]{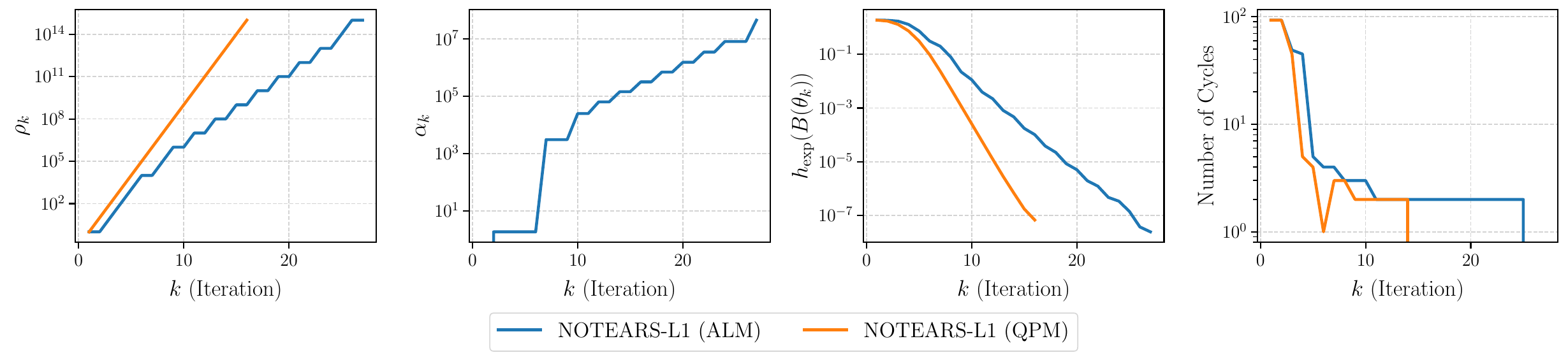}
} \\
\subfloat[NOTEARS-MLP.]{
  \includegraphics[width=0.97\textwidth]{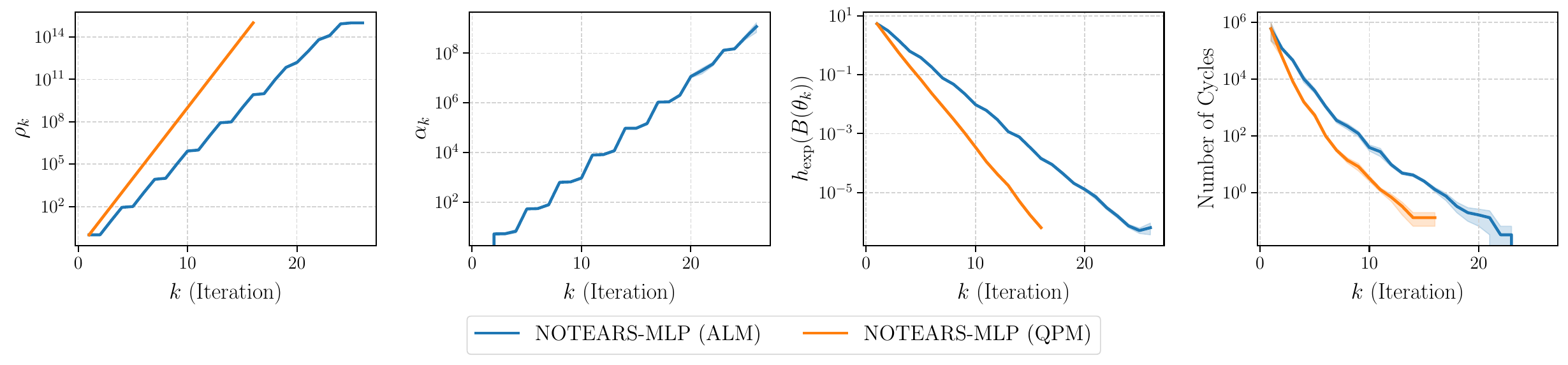}
}\\
\subfloat[NOTEARS-MLP-L1.]{
  \includegraphics[width=0.97\textwidth]{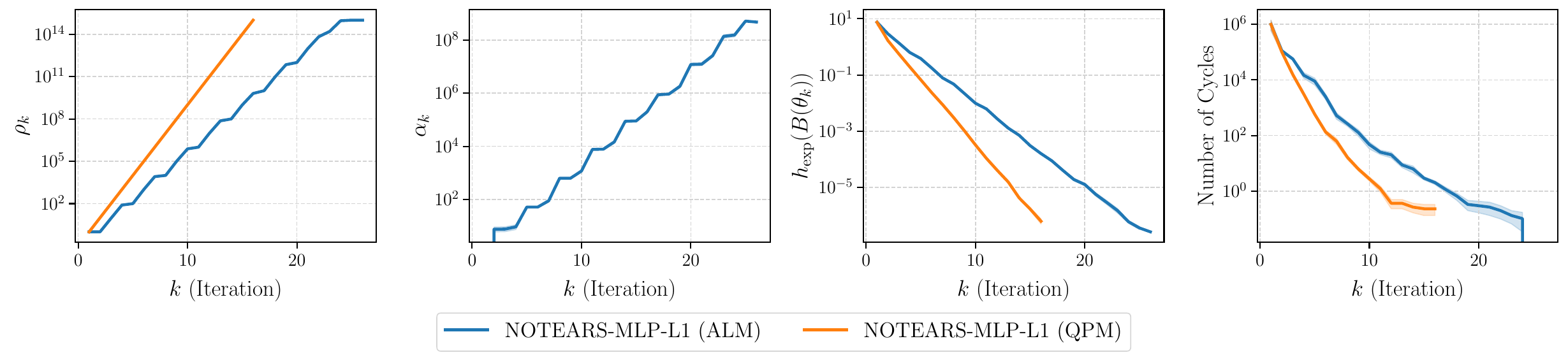}
}
\caption{Optimization processes of the continuous constrained formulation using ALM and QPM on a real dataset by \citet{Sachs2005causal}. Each data point corresponds to the $k$-th iteration. Shaded area denotes standard errors over $30$ trials.}
\label{fig:ALM_vs_QPM_optz_process_sachs}
\end{figure}

\begin{figure}[!t]
\centering
\subfloat[NOTEARS.]{
  \includegraphics[width=0.97\textwidth]{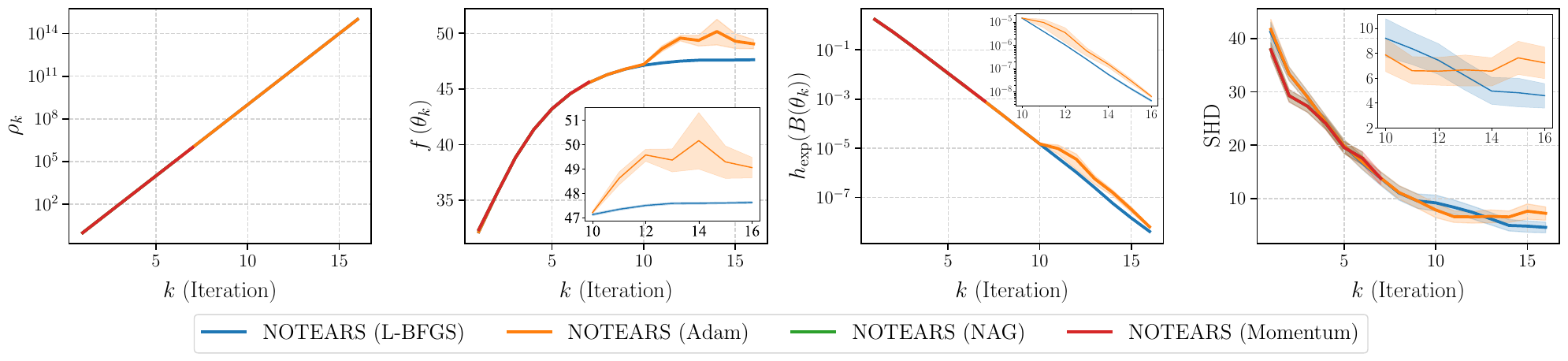}
} \\
\subfloat[NOTEARS-MLP-L1.]{
  \includegraphics[width=0.97\textwidth]{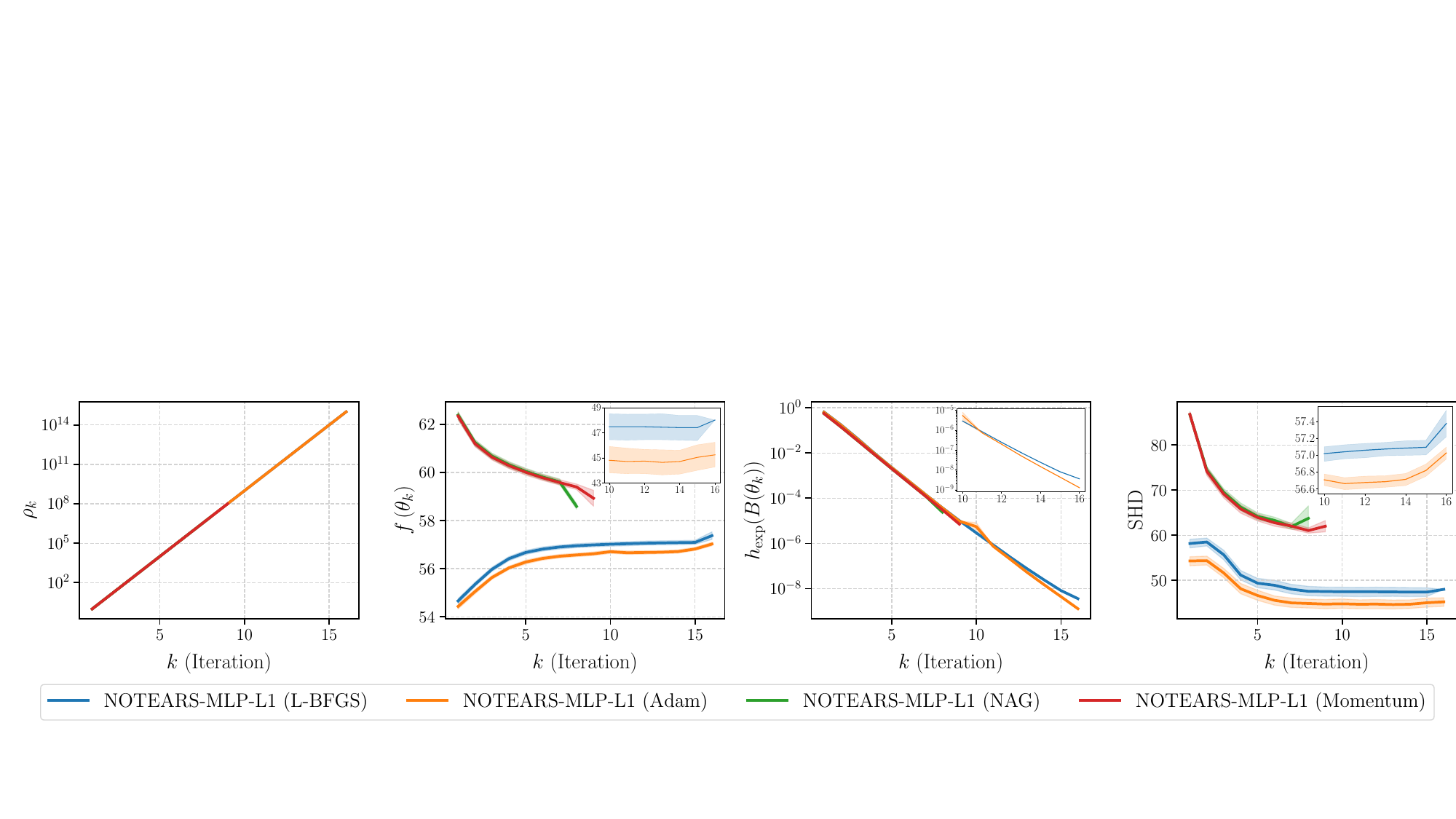}
}
\caption{Optimization processes of different optimization algorithms for solving the QPM subproblems of NOTEARS and NOTEARS-MLP-L1 on synthetic data. The ground truths are $100$-node ER1 graphs and the sample size is $n=1000$. Each data point corresponds to the $k$-th iteration. Shaded area denotes standard errors over $30$ trials. The blue line overlaps with the orange line in the first panel.}
\label{fig:optimizer_optz_process}
\end{figure}

\begin{figure}[!t]
\centering
\subfloat[NOTEARS with ER1 graphs.]{
  \includegraphics[width=0.97\textwidth]{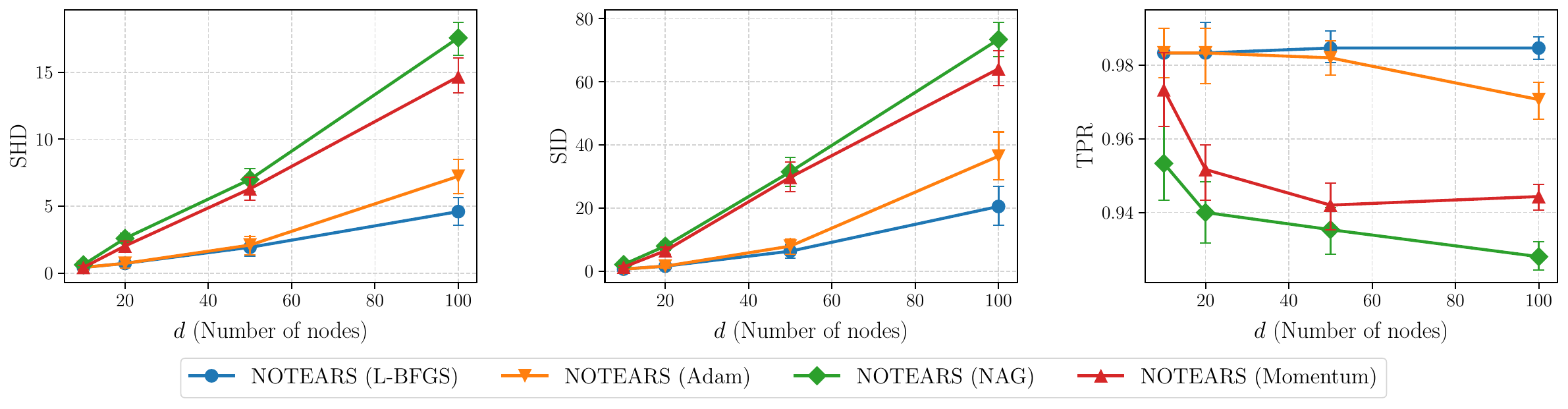}
} \\
\subfloat[NOTEARS with SF4 graphs.]{
  \includegraphics[width=0.97\textwidth]{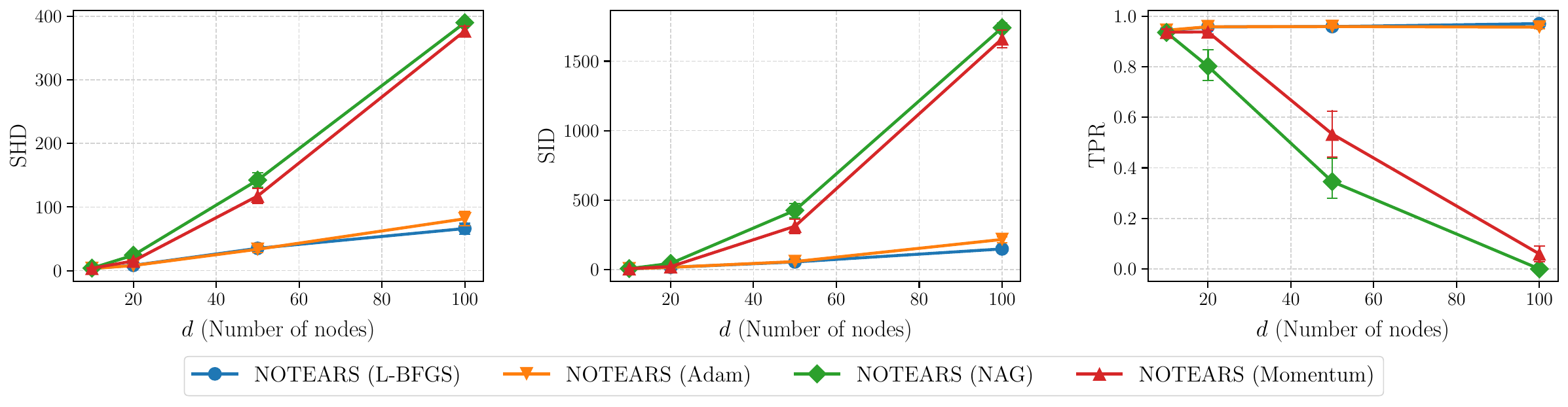}
} \\
\subfloat[NOTEARS-MLP-L1 with ER1 graphs.]{
  \includegraphics[width=0.97\textwidth]{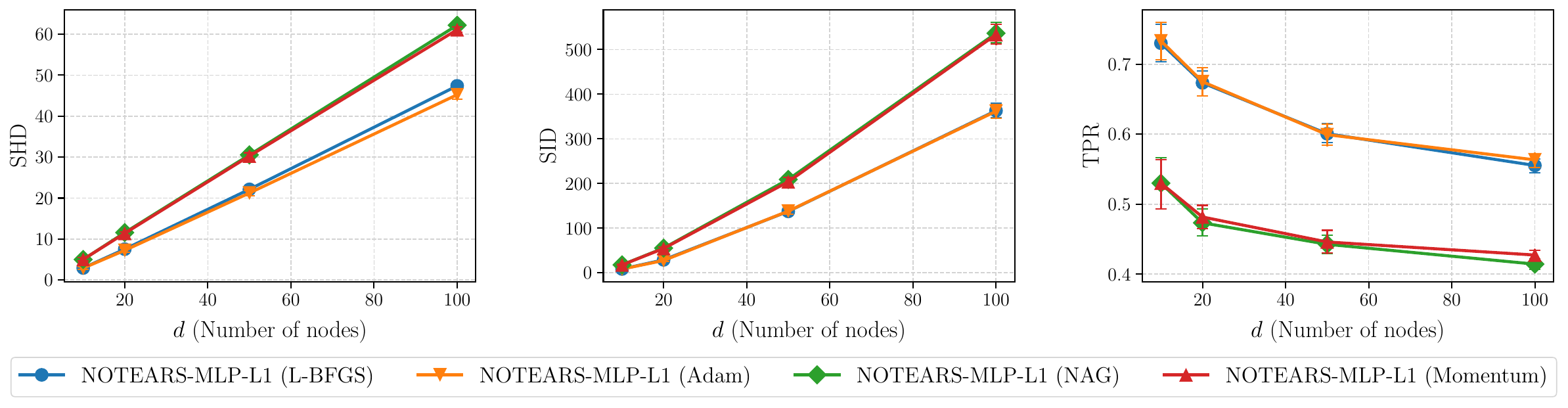}
} \\
\subfloat[NOTEARS-MLP-L1 with SF4 graphs.]{
  \includegraphics[width=0.97\textwidth]{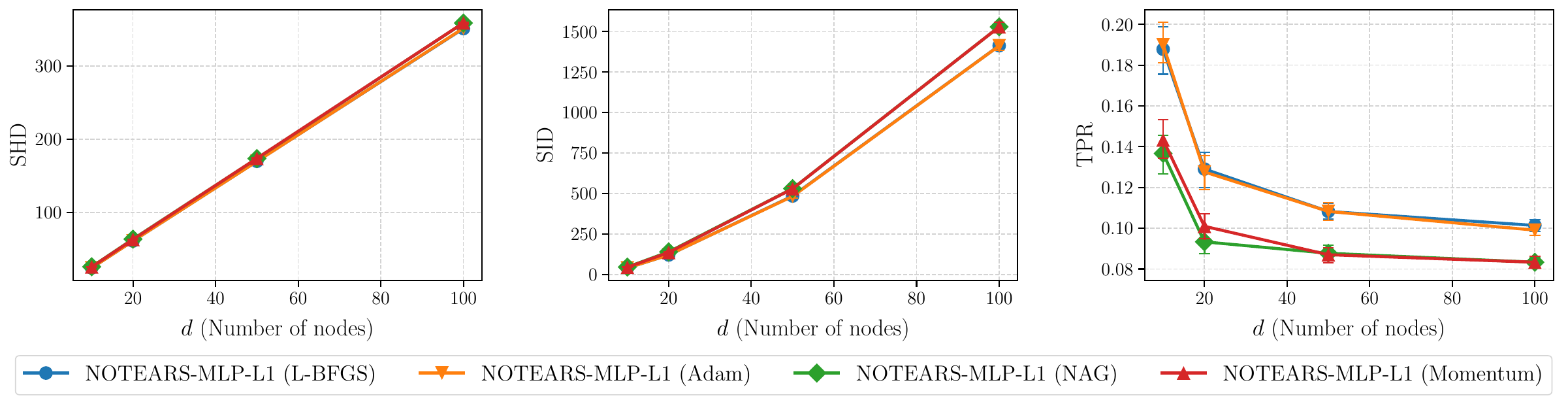}
}
\caption{Empirical results of different optimization algorithms for solving the QPM subproblems of NOTEARS and NOTEARS-MLP-L1 on synthetic data. The sample size is $n=1000$. Lower is better, except for TPR. Error bars denote standard errors over $30$ trials.}
\label{fig:optimizer_results_2}
\end{figure}

\begin{figure}[!t]
\centering
\subfloat[ER1 graphs.]{
  \includegraphics[width=0.97\textwidth]{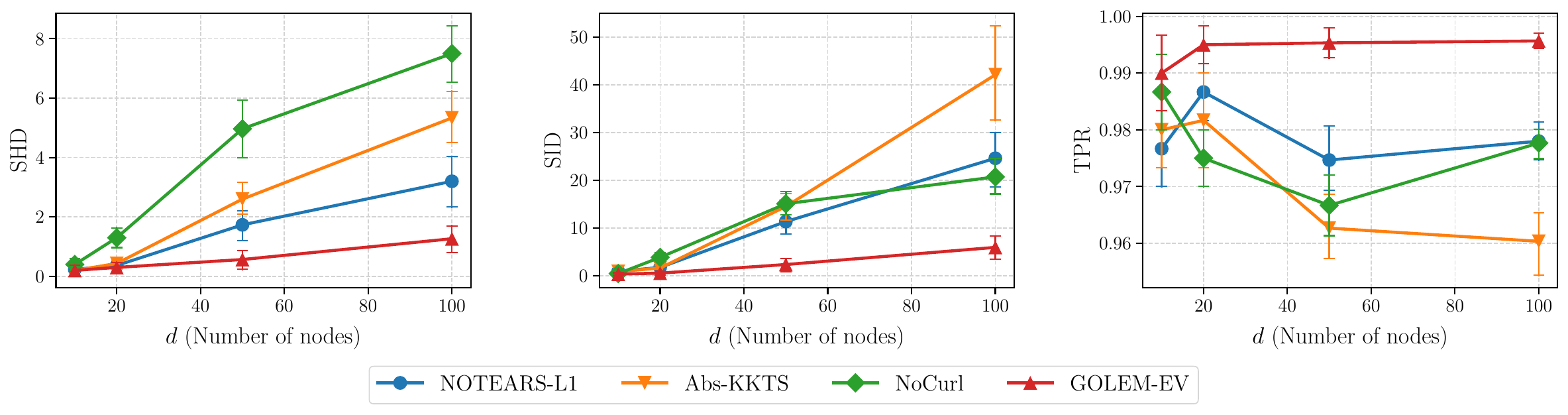}
} \\
\subfloat[SF4 graphs.]{
  \includegraphics[width=0.97\textwidth]{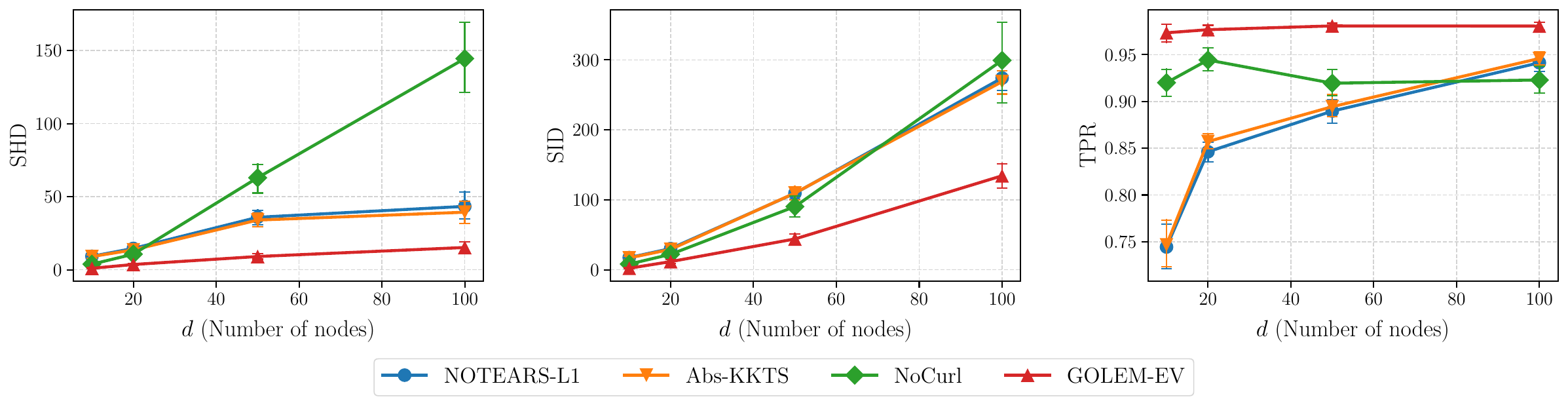}
}
\caption{Empirical results of different structure learning methods on synthetic data with sample size $n=1000$. Lower is better, except for TPR. Error bars denote standard errors over $30$ trials.}
\label{fig:with_baselines_1000_samples}
\end{figure}

\begin{figure}[!t]
\centering
\subfloat[ER1 graphs.]{
  \includegraphics[width=0.97\textwidth]{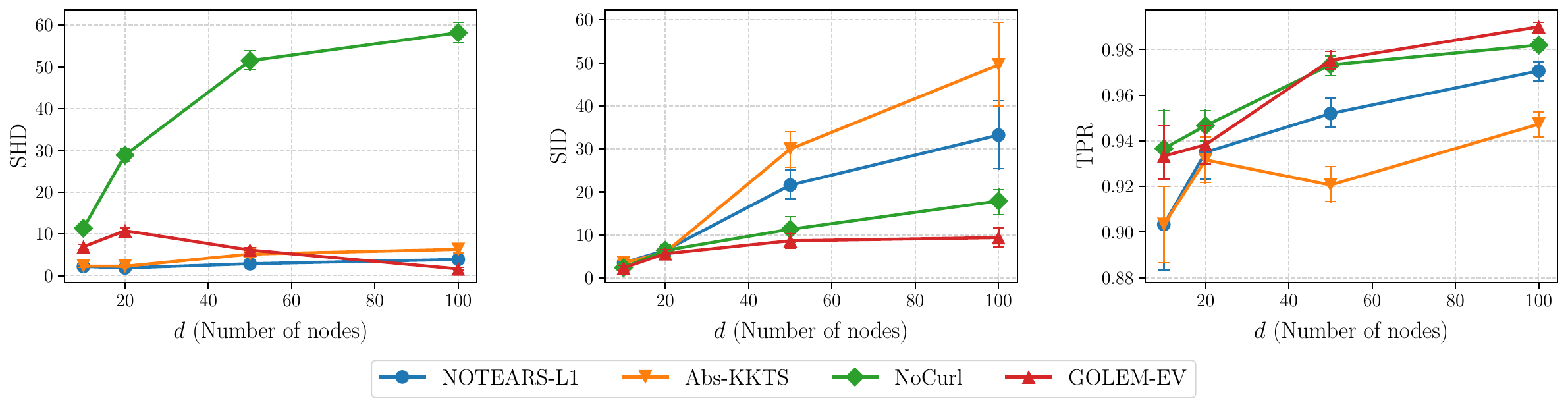}
} \\
\subfloat[SF4 graphs.]{
  \includegraphics[width=0.97\textwidth]{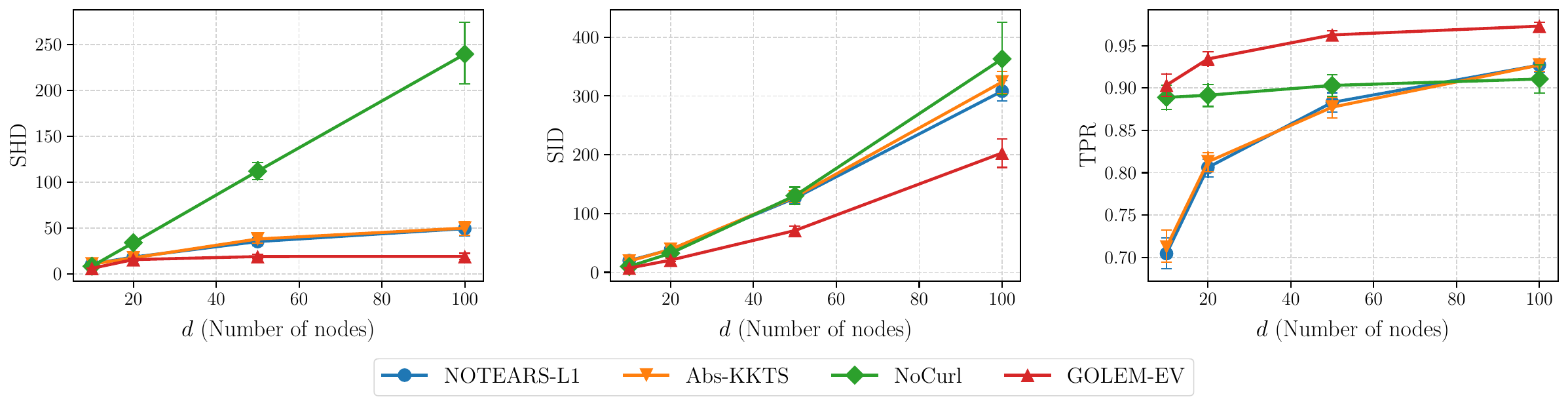}
}
\caption{Empirical results of different structure learning methods on synthetic data with sample size $n=3d$. Lower is better, except for TPR. Error bars denote standard errors over $30$ trials.}
\label{fig:with_baselines_3d_samples}
\end{figure}